\documentclass[letterpaper,journal]{IEEEtran}

\usepackage{cite}
\usepackage{amsmath,amssymb,amsfonts}
\usepackage{amsthm}
\usepackage{aliascnt}
\usepackage{array}
\usepackage{graphicx}
\usepackage{textcomp}
\usepackage{xcolor}
\usepackage[hidelinks]{hyperref}
\usepackage{booktabs}
\usepackage{orcidlink}
\usepackage{algorithm}
\usepackage{algorithmic}
\usepackage{tikz}
\usetikzlibrary{arrows.meta, positioning, shapes.symbols, shadows.blur}

\graphicspath{{figures/}}

\definecolor{techblue}{HTML}{1A5F7A}
\definecolor{techbg}{HTML}{F0F8FF}
\definecolor{techgray}{HTML}{4A5568}

\newtheorem{theorem}{Theorem}[section]

\newaliascnt{lemma}{theorem}
\newtheorem{lemma}[lemma]{Lemma}
\aliascntresetthe{lemma}

\newaliascnt{corollary}{theorem}
\newtheorem{corollary}[corollary]{Corollary}
\aliascntresetthe{corollary}

\newaliascnt{definition}{theorem}
\newtheorem{definition}[definition]{Definition}
\aliascntresetthe{definition}

\newaliascnt{proposition}{theorem}
\newtheorem{proposition}[proposition]{Proposition}
\aliascntresetthe{proposition}

\newaliascnt{assumption}{theorem}
\newtheorem{assumption}[assumption]{Assumption}
\aliascntresetthe{assumption}

\newaliascnt{remark}{theorem}
\newtheorem{remark}[remark]{Remark}
\aliascntresetthe{remark}

\newif\ifforReview
\forReviewtrue
\newif\ifrevision\revisionfalse

\begin{document}

\title{Conflict-Predictive Variable Horizons in Multi-Drone Distributed Model Predictive Control}

\author{
    Linda~M\"{u}mken\orcidlink{0009-0005-2646-9397},
    Michael~Schwung\orcidlink{0000-0002-8420-737X},
    Stefan~Lier\orcidlink{0000-0002-3314-7610},
    and~Andreas~Schwung\orcidlink{0000-0001-8405-0977}
    \thanks{The research project was funded by ``The Ministry of the Environment,
        Nature Conservation and Transport of the State of North Rhine-Westphalia''
        in Germany and co-financed by the European Union for the research project
        ``SIDDA -- Sustainable Intermodal Drone Delivery Airline'' with grant
        number IN-ML-1-013b.}
    \thanks{L. M\"{u}mken and A. Schwung are with the Department of Automation
        Technology, South Westphalia University of Applied Sciences, 59494 Soest,
        Germany (e-mail: muemken.linda@fh-swf.de; schwung.andreas@fh-swf.de).}
    \thanks{M. Schwung is with the Institute of Automation and Computer Control,
        Ruhr University Bochum, 44801 Bochum, Germany
        (e-mail: michael.schwung@rub.de).}
    \thanks{S. Lier is with the Department of Logistics and Supply Chain
        Management, South Westphalia University of Applied Sciences, 59872
        Meschede, Germany (e-mail: lier.stefan@fh-swf.de).}
}

\maketitle

\begin{abstract}
    In distributed model predictive control for multi-drone collision avoidance, a fixed prediction horizon forces a compromise: a short horizon is inexpensive but reacts late to approaching neighbors, whereas a long one anticipates conflicts at a per-step cost that grows superlinearly with its length. We propose a conflict-predictive variable horizon that each drone sets locally, leaving the distributed model predictive control itself unchanged. From a short history of observed positions, a drone extrapolates the flight lines of its neighbors, tests each against its own using confidence funnels that narrow with prediction range, and obtains each time to conflict in closed form. The horizon is then the smallest admissible value whose planning window covers the farthest predicted conflict. It collapses to its minimum in clear airspace and grows only when a conflict lies ahead. Provided this minimum meets a single computable feasibility bound, we prove that recursive feasibility and asymptotic stability are preserved for every horizon the policy can select. These guarantees hold for a linear model, and a cascaded inner loop reduces each quadrotor's translational dynamics to a perturbed double integrator, so they carry over to the linearized quadrotor model and, as practical stability, to the full nonlinear one. In simulation on dense antipodal-swap benchmarks, the variable horizon reduces both per-step solver cost and total computation well below those of a long fixed horizon, and it maintains separation in every run, which a short fixed horizon of comparable per-step cost does not.
\end{abstract}


\section{Introduction}\label{sec:intro}

    Low-altitude airspace is becoming increasingly crowded with autonomous drones used for delivery, inspection, and monitoring, and traffic management frameworks anticipate dense beyond-visual-line-of-sight operations in shared airspace volumes~\cite{sesar_uspace_conops, kopardekar2016utm}. Distributed model predictive control (DMPC) is a natural fit for such traffic, because each drone repeatedly optimizes its own trajectory over a finite prediction horizon from the observed states of its neighbors and enforces separation as a constraint~\cite{luis2020dmpc_drones, stomberg2022dmpc_compendium}. The prediction horizon is the pivotal parameter. It must reach far enough ahead to resolve an approaching conflict before it becomes unavoidable, yet the per-step optimization cost grows superlinearly with its length. Conflicts, however, are intermittent and spatially sparse. A drone spends most of its mission in clear airspace and encounters others only briefly, yet a fixed horizon must be sized for the worst case, so its cost falls on every step even when no neighbor is near.
    This paper aims to design a controller that only pays for anticipation when it predicts an impending conflict, ensuring that its safety does not depend on the accuracy of the prediction.

    So far, the horizon has been treated as a constant, sized once for the operating density and then paid for at every step~\cite{drones10020139, adaptive_dmpc_2026}. Dynamic-horizon schemes do exist, but none sizes the horizon to a predicted conflict while keeping the feasibility and stability of the underlying DMPC independent of that prediction (\autoref{sec:related}).

    We close this gap with a conflict-predictive variable horizon that each drone sets locally. At every step, a lightweight linear predictor extrapolates each neighbor's flight line from a short observation history, and the drone tests that line against its own using a confidence funnel, a tolerance tube that narrows as prediction confidence decays. One closed-form evaluation per neighbor yields the time to conflict, and the horizon becomes the smallest admissible value whose planning window covers the farthest of these times. In clear airspace it stays at its minimum $H_{\min}$, and it reaches its maximum $H_{\max}$ only when a conflict is flagged at the far end of the window, where the funnel is narrowest and only a close pass still counts as an overlap.

    The resulting controller decouples its guarantees from the conflict prediction that drives it. We prove that feasibility and asymptotic stability hold for every horizon the policy can select, provided $H_{\min}$ meets the computable feasibility bound $H_{\min}^{\mathrm{feas}}$, all without modifying the distributed controller itself. This rests on a property we establish in the analysis. The horizon enters only through a finite-horizon approximation residual, which we show stays uniformly bounded over $[H_{\min}, H_{\max}]$ and is absorbed by a certificate depending solely on the drone's physical state. A misprediction therefore blunts anticipation and costs efficiency, never safety. To our knowledge, no other dynamic-horizon scheme decouples its guarantees from the adaptation signal in this way. The same perturbation argument carries the guarantees to linearized quadrotor dynamics and yields practical stability for the full nonlinear model.

We make the following contributions:
\begin{itemize}
    \item We introduce a variable-horizon DMPC that preserves recursive feasibility and asymptotic stability for every admissible horizon, under one computable condition on $H_{\min}$.
    \item We extend these guarantees to quadrotor dynamics, asymptotic for the linearized model and practical for the full nonlinear one.
    \item We show on dense antipodal-swap benchmarks that the variable horizon undercuts a long fixed horizon in per-step and total computation, completes swaps it cannot, and holds separation where a short fixed horizon of comparable cost fails.
\end{itemize}

    The remainder of the paper is organized as follows. \autoref{sec:related} reviews related work on horizon adaptation and distributed MPC for collision avoidance. \autoref{sec:framework} presents the conflict-predictive variable-horizon controller. \autoref{sec:theory} establishes recursive feasibility and asymptotic stability under the variable horizon. \autoref{sec:extended_models} extends the guarantees to linearized and nonlinear quadrotor dynamics. \autoref{sec:evaluation} reports the experimental evaluation, and \autoref{sec:conclusion} concludes.

\section{Related Work}\label{sec:related}
    We review the two lines of work on which our scheme builds and set out where it departs from each. The first covers online adaptation of the prediction horizon and the neighbor-motion prediction that feeds our policy. The second covers distributed MPC for multi-drone collision avoidance.

\subsection{Variable Prediction Horizons}\label{ssec:rw_horizons}
    The prediction horizon fixes both the reach and the per-step cost of the optimization. Many methods adapt it online. What drives that adaptation is what separates these methods from our horizon policy.

    One option is to treat the horizon length as a decision variable of the optimization itself~\cite{michalska1993robust, richards2006variablehorizon}. This couples the horizon to the objective and typically turns the per-step problem into a non-convex or mixed-integer program, which inflates the very computation the adaptation is meant to save. A less invasive alternative keeps the horizon outside the optimization and adapts it through a stabilizing heuristic, growing or contracting it against a control-Lyapunov or terminal-set criterion on the controlled system's own state~\cite{krener2018ahmpc, sun2019selftriggered_adaptivehorizon}. A third line avoids adaptation altogether and fixes one long horizon that approximates the infinite-horizon cost~\cite{ROSTAMI2023100881}, and by construction pays for that horizon at every step. Where these schemes adapt, they size the horizon to a property of the controlled system itself rather than, as in our work, to the geometry of an approaching neighbor, and they contract it toward the goal where ours expands it toward a conflict.

    Data-driven schemes go further and learn the horizon, mapping the current state to a length that trades tracking against computation~\cite{bohn2021rlhorizon}. In the multi-robot case, a learned policy sets a per-robot horizon, and an on-demand rule preserves the avoidance constraints when that horizon shrinks below the time to a detected conflict~\cite{gupta2023vodca}. Our policy reverses the direction of this coupling. A conflict prediction is not a correction applied to a horizon set elsewhere but the input that raises the horizon until the planning window reaches the encounter, and the safety argument never depends on that prediction being correct. A learned policy also carries no closed-form guarantee, so feasibility and stability must be argued separately and the training repeated per task and per density.

    The most closely related methods are event- and self-triggered schemes, which only recompute the control when a state or error condition demands it~\cite{GRAFE202279, ma2021etdmpc_varhorizon, bezier_trajectory}. In these methods, the trigger determines when to recompute and which subsystem performs the computation. If a horizon adapts, it contracts toward a terminal set. In contrast, our policy determines how far to look ahead and expands the horizon just as far as the time to the farthest predicted encounter requires and contracts it again once the airspace is clear.

    Across these lines of work, the horizon is driven by the controlled system's own state while ours is driven by a predicted encounter. It enters the stability analysis only through a residual that stays uniformly bounded over the admissible band.

    To anticipate conflicts, a controller needs a model of where its neighbors will go. The traditional approach is constant-velocity extrapolation, which is employed in velocity obstacles and their reciprocal variants to project each drone along its current heading~\cite{fiorini1998velocity, vandenberg2011rvo}. More sophisticated predictors offer better anticipation at a higher cost. Gaussian process and learned motion forecasts capture maneuvering neighbors~\cite{olcay2024dynamic}, and chance-constrained stochastic MPC pushes the predicted uncertainty into the separation constraint to enforce a probabilistic safety margin~\cite{yoshikawa2023chance}. In all of these, forecast accuracy is safety-critical, since an inaccurate forecast enters the avoidance constraint directly and can void the separation it is meant to enforce. We take the opposite stance and use a deliberately lightweight linear predictor wrapped in confidence funnels. Because feasibility and stability rest on a prediction-independent lower bound on the horizon, a misprediction degrades anticipation only. The funnel construction pursues the trajectory-prediction direction left open by the framework build from~\cite{adaptive_dmpc_2026}, and it admits a learned predictor without affecting the guarantees.

\subsection{Distributed MPC for Collision Avoidance}\label{ssec:rw_dmpc}

    Distributed MPC solves the coupled multi-drone problem by letting each drone optimize its own trajectory based on its neighbors' latest plans~\cite{kuwata2007distributed, luis2020dmpc_drones, shorinwa2023dmpc}, commonly through dual decomposition or the Alternating Direction Method of Multipliers (ADMM) with Gauss--Seidel sweeps~\cite{stomberg2022dmpc_compendium, bertsekas1989parallel}. Stability without terminal ingredients is then recovered either by contraction and self-organization arguments~\cite{koehler2022sequential_dmpc} or by the long-horizon route described above. Set-based safety encodings, such as control barrier functions and their distributed and learned multi-agent extensions~\cite{ames2017cbf, zeng2021mpccbf, zhang2025gcbf}, enforce separation through a forward-invariant set rather than a receding-horizon constraint, and are complementary to the horizon question studied here.

    We build on two foundations, namely the velocity-dependent safety sphere and the sphere-packing bounds that cap the attainable drone density resulting from the geometric and control-theoretic capacity analysis in~\cite{drones10020139}. The adaptive safety zone DMPC, together with its fixed-horizon feasibility and stability guarantees, comes from the framework in~\cite{adaptive_dmpc_2026}, which combines a contraction argument with a Lyapunov function based on physical energy and an asynchronous Gauss--Seidel iteration. Our work turns the horizon, the central cost parameter of this framework, over to an online conflict-predictive policy. The step that makes this admissible is our contribution. We show that the single horizon-dependent residual in the analysis is uniformly bounded over the admissible band, so the certificate survives every horizon the policy selects. We then extend the resulting guarantees beyond the double integrator to linearized and full nonlinear quadrotor dynamics. The adaptive radius and the underlying formulation carry over unchanged.

\section{Conflict-Predictive Variable-Horizon Control}\label{sec:framework}

    Consider $N$ drones sharing a bounded airspace $\Omega \subset \mathbb{R}^3$, indexed by $i \in \{1, \ldots, N\}$, each steering toward a goal $\bar{\mathbf{p}}_i \in \Omega$. Each drone observes time-stamped positions of nearby drones through onboard sensing or broadcast and stores them in a short rolling history. All neighbor quantities, including velocities and adaptive radii, are computed from these positions. The DMPC enforces collision avoidance through velocity-dependent safety zones and pairwise separation constraints. Its central parameter is the prediction horizon $H$, which carries the short-versus-long trade-off of \autoref{sec:intro}. A fixed horizon commits to one side of that trade-off for the entire mission. \autoref{fig:overview} shows the proposed scheme across the drones, with the local loop for every drone. The problem statement below specifies the goals for the structure.

    \begin{figure}[t]
        \centering
        \scalebox{0.95}{\input{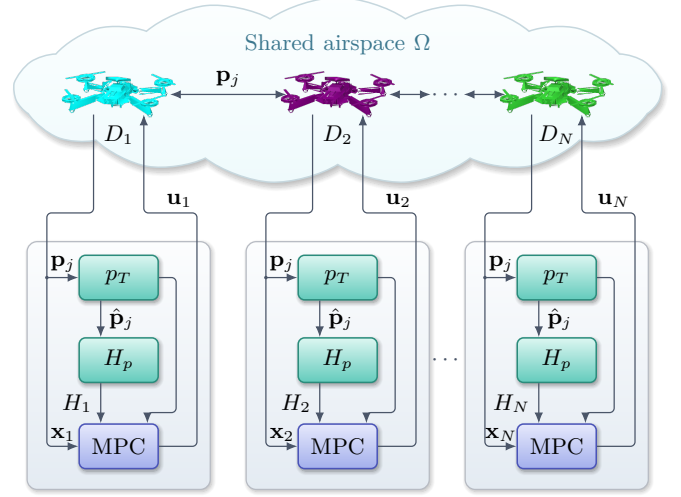}}
        \caption{Proposed scheme on $N$ drones sharing the airspace $\Omega$. Each drone observes the neighbor positions $\mathbf{p}_j$ and runs the same local loop. The predictor $p_T$ turns the observations into flight lines $\hat{\mathbf{p}}_j$, the policy $H_p$ sets the horizon $H_i$ from the funnel conflict test, and the DMPC turns the state $\mathbf{x}_i$ into the control input $\mathbf{u}_i$. The horizon is the only quantity that the scheme changes in the otherwise unchanged DMPC. All signals are taken at step $k$.}
        \label{fig:overview}
    \end{figure}

    \noindent\textbf{Problem.}
    Design, for each drone $i$, a control law $\{\mathbf{u}_i(k)\}$ and a per-step prediction horizon $H_i(k)$ such that:
    \begin{enumerate}
        \item[\textbf{(S)}] \emph{Safety:} the pairwise separation $\|\mathbf{p}_i(t) - \mathbf{p}_j(t)\| \geq r_i(t) + r_j(t)$ and the airspace constraint $B_{r_i(t)}(\mathbf{p}_i(t)) \subset \Omega$, with $B_r(\mathbf{p})$ the closed ball of radius $r$ centered at $\mathbf{p}$, hold for all $i \neq j$ and $t \geq 0$.
        \item[\textbf{(St)}] \emph{Stability:} each drone converges to its goal, $\|\mathbf{p}_i(t) - \bar{\mathbf{p}}_i\| \to 0$ as $t \to \infty$.
        \item[\textbf{(E)}] \emph{Efficiency:} the horizon, and with it the per-step optimization cost, is minimal among the admissible horizons whose planning window covers every predicted conflict. In particular, $H_i(k) = H_{\min}$ whenever no conflict is predicted.
    \end{enumerate}

    The controller is distributed by design. Each drone solves its own problem based on the observed states of its neighbors, so the three requirements must be met locally, without central coordination. The central difficulty lies in reconciling \textbf{(E)} with \textbf{(S)} and \textbf{(St)}. A short horizon reduces computation but may prevent the optimizer from resolving an approaching conflict within its planning window. A long horizon preserves safety but sacrifices the efficiency gain. A reactively chosen horizon must therefore never fall below the value that keeps the controller feasible and stable, and it must still collapse to a small value whenever the airspace around the drone is clear.

    We resolve this by separating the two roles of the horizon. A fixed lower bound $H_{\min}$, verified once before deployment, carries feasibility and stability. The conflict prediction then selects a horizon within the admissible band $[H_{\min}, H_{\max}]$, whose upper bound caps the worst-case per-step computation. The prediction therefore governs anticipation alone and cannot weaken the guarantees. A horizon resting at the lower bound remains safe, not because the airspace is clear, but because $H_{\min}$ itself meets the feasibility bound.

    \autoref{ssec:preliminaries} recalls the dynamics, the adaptive safety zone, and the DMPC formulation on which the scheme builds. \autoref{sec:theory} then establishes the guarantees of the closed loop and condenses them into \autoref{thm:var_guarantees}.

\subsection{Preliminaries}\label{ssec:preliminaries}

\subsubsection*{Dynamics}
Following the inner-loop time-scale separation of a cascaded quadrotor controller~\cite{mahony2012multirotor, khalil_nonlinear_2002}, the translational motion of each drone reduces to a double integrator with state $\mathbf{x}_i = (\mathbf{p}_i, \mathbf{v}_i) \in \mathbb{R}^3 \times \mathbb{R}^3$. With sampling time $\Delta t > 0$, the discretized dynamics over a horizon step are
\begin{align}\label{eq:dynamics}
    \mathbf{p}_i(s+1) &= \mathbf{p}_i(s) + \Delta t\, \mathbf{v}_i(s) + \tfrac{1}{2}\Delta t^2\, \mathbf{u}_i(s), \\
    \mathbf{v}_i(s+1) &= \mathbf{v}_i(s) + \Delta t\, \mathbf{u}_i(s), \nonumber
\end{align}
subject to the acceleration and speed limits $\|\mathbf{u}_i(s)\| \leq U_{\max}$ and $\|\mathbf{v}_i(s)\| \leq V_{\max}$, where $U_{\max}$ represents the physical propulsion limits. The extension of all results to the linearized and to the full nonlinear rigid-body quadrotor model is established in \autoref{sec:extended_models}.

\subsubsection*{Adaptive safety zone}
Each drone occupies a closed safety sphere $B_{r_i(t)}(\mathbf{p}_i)$ whose radius scales with the braking distance~\cite{drones10020139},
\begin{align}\label{eq:adaptive_radius}
    r_i(\mathbf{v}_i) = r_{\min} + \alpha\, \frac{\|\mathbf{v}_i\|^2}{2 U_{\max}},
\end{align}
with base radius $r_{\min} > 0$ and adaptation parameter $\alpha \in (0,1]$. The radius is tight at low speed and expands at high speed, so that slower drones occupy less airspace. Safety requires the pairwise separation
\begin{align}\label{eq:separation}
    \|\mathbf{p}_i(t) - \mathbf{p}_j(t)\| \geq r_i(t) + r_j(t), \quad \forall\, i \neq j,\ \forall\, t.
\end{align}

\subsubsection*{Distributed model predictive control}
At each decision time $k\Delta t$, with $s \triangleq k + h$ for $h \in \{0, \ldots, H-1\}$, drone $i$ solves its own finite-horizon problem, minimizing $\sum_{h=0}^{H-1}[\ell(\mathbf{x}_i(s), \mathbf{u}_i(s)) + \lambda\|\mathbf{v}_i(s)\|^2]$ subject to the dynamics~\eqref{eq:dynamics}, the actuation and speed bounds, the airspace constraint $B_{r_i(s)}(\mathbf{p}_i(s)) \subset \Omega$, and the separation constraint~\eqref{eq:separation} in which the neighbor positions are the predicted trajectories $\hat{\mathbf{p}}_j$, $j \in \mathcal{N}_i$. The stage cost $\ell$ penalizes goal deviation and control effort, and $\lambda > 0$ weights a velocity penalty that shrinks safety zones in congested regions. The coupled problem over all drones is solved by asynchronous ADMM--Gauss--Seidel iteration~\cite{stomberg2022dmpc_compendium, bertsekas1989parallel}, each drone optimizing in turn from the predicted lines $\hat{\mathbf{p}}_j$ it has formed from its own observations. This DMPC formulation and its fixed-horizon feasibility and stability guarantees have been presented in~\cite{adaptive_dmpc_2026}. We now extend the results by adapting the horizon $H$.

\subsection{Neighbor History and Linear Prediction}\label{ssec:prediction}

Each drone $i$ maintains, for every observed neighbor $j$, a rolling history of its measured positions over the most recent $L$ samples,
\begin{align}\label{eq:history}
    \mathbf{p}_j(k{-}L{+}1{:}k) \triangleq \big(\mathbf{p}_j(k{-}L{+}1), \ldots, \mathbf{p}_j(k)\big).
\end{align}
A first-order least-squares fit of $\mathbf{p}_j(k{-}L{+}1{:}k)$ yields the current position estimate $\mathbf{p}_j$ and a constant-velocity flight line
\begin{align}\label{eq:linear_pred}
    \hat{\mathbf{p}}_j(\tau) = \mathbf{p}_j + \tau\, \tilde{\mathbf{v}}_j, \qquad \tau \in [0, t_{\max}],
\end{align}
over the prediction window $t_{\max} = H_{\max}\Delta t$, where $H_{\max}$ is the maximum admissible horizon introduced in \autoref{ssec:horizon}. The predicted velocity inherits its direction and speed from the fitted line, with the speed bounded below at a fraction of the speed limit,
\begin{align}\label{eq:speed_floor}
    \tilde{\mathbf{v}}_j = \hat{\mathbf{d}}_j \cdot \max\!\big(\hat{v}_j,\ \nu\, V_{\max}\big), \qquad \nu \in (0, 1],
\end{align}
with $\hat{\mathbf{d}}_j$ the unit direction and $\hat{v}_j$ the fitted speed. No velocity measurement of the neighbor is needed. The lower bound $\nu V_{\max}$ keeps the prediction conservative when a neighbor is momentarily slow or at rest, so that latent conflicts are not masked by a vanishing relative velocity.
The ego drone's own line $\hat{\mathbf{p}}_i(\tau) = \mathbf{p}_i + \tau\, \tilde{\mathbf{v}}_i$ is formed analogously, with $\hat{\mathbf{d}}_i$ taken from its reference toward $\bar{\mathbf{p}}_i$. The history length $L$ is a design parameter. The present implementation uses $L = 50$ and computes the fit~\eqref{eq:history}--\eqref{eq:linear_pred} from the first observed position onward. The extrapolation is deliberately lightweight, requiring no model of a neighbor's controller and admitting the closed-form conflict test of \autoref{ssec:detection}. The ego line uses the reference direction rather than the drone's own previously planned MPC trajectory as it spans only the previous horizon. Coordination therefore rests on observed positions alone: no drone transmits its planned trajectory or goal, and its mission intent stays private.

\subsection{Confidence Funnels}\label{ssec:funnel}

Linear extrapolation is reliable in the near term but degrades as $\tau$ grows, because neighbors maneuver and the constant-velocity assumption decays. We therefore propose to surround each predicted flight line with a confidence funnel, a tube that narrows with prediction time and whose radius sets the conflict-test tolerance. The funnel weighs the uncertainty of the predicted trajectory against the criticality of that uncertainty. Close to the current time stamp, the prediction is reliable and an encounter leaves little time to react, so the tolerance is wide. Far ahead, the prediction is uncertain, but this uncertainty is benign, because the encounter is distant and the test is repeated at every step, so only a severe predicted geometry flags a conflict. Let the progress along the prediction window be $\sigma(\tau) = \tau / t_{\max} \in [0,1]$ and let the initial radius at the funnel's wide end be the maximum safety radius $r_{\max} = r_{\min} + \alpha V_{\max}^2/(2U_{\max})$, i.e., the largest value the adaptive radius $r_i(\mathbf{v}_i)$ attains.

\begin{definition}[Confidence funnel radius]\label{def:funnel}
The confidence funnel radius along a predicted flight line is
\begin{align}\label{eq:funnel}
    \psi_{ij}(\tau) = r_{\mathrm{f}} + \big(r_{\max} - r_{\mathrm{f}}\big)\, \exp\!\Big(\!-\frac{\sigma(\tau)}{c_{\mathrm{d}}}\Big),
\end{align}
where the funnel lower bound $r_{\mathrm{f}} = r_j(\hat{v}_j)$ is the neighbor's current adaptive radius~\eqref{eq:adaptive_radius} evaluated at the fitted speed, and $c_{\mathrm{d}} > 0$ is the decay constant.
\end{definition}

The funnel forms a cone that narrows along the neighbor's flight line, so the tolerance is wide and conservative near the present and tight and selective far ahead, as illustrated in \autoref{fig:funnel}. The lower bound depends on the neighbor's state: a slow neighbor induces a tight far-future tolerance, whereas a fast neighbor, whose radius approaches $r_{\max}$, retains a nearly uniform tube. The decay constant $c_{\mathrm{d}}$ trades sensitivity against false alarms: small $c_{\mathrm{d}}$ discounts distant predictions aggressively, whereas large $c_{\mathrm{d}}$ keeps the tube wide across the whole window.

\begin{figure}[t]
    \centering
    \includegraphics[width=0.8\columnwidth]{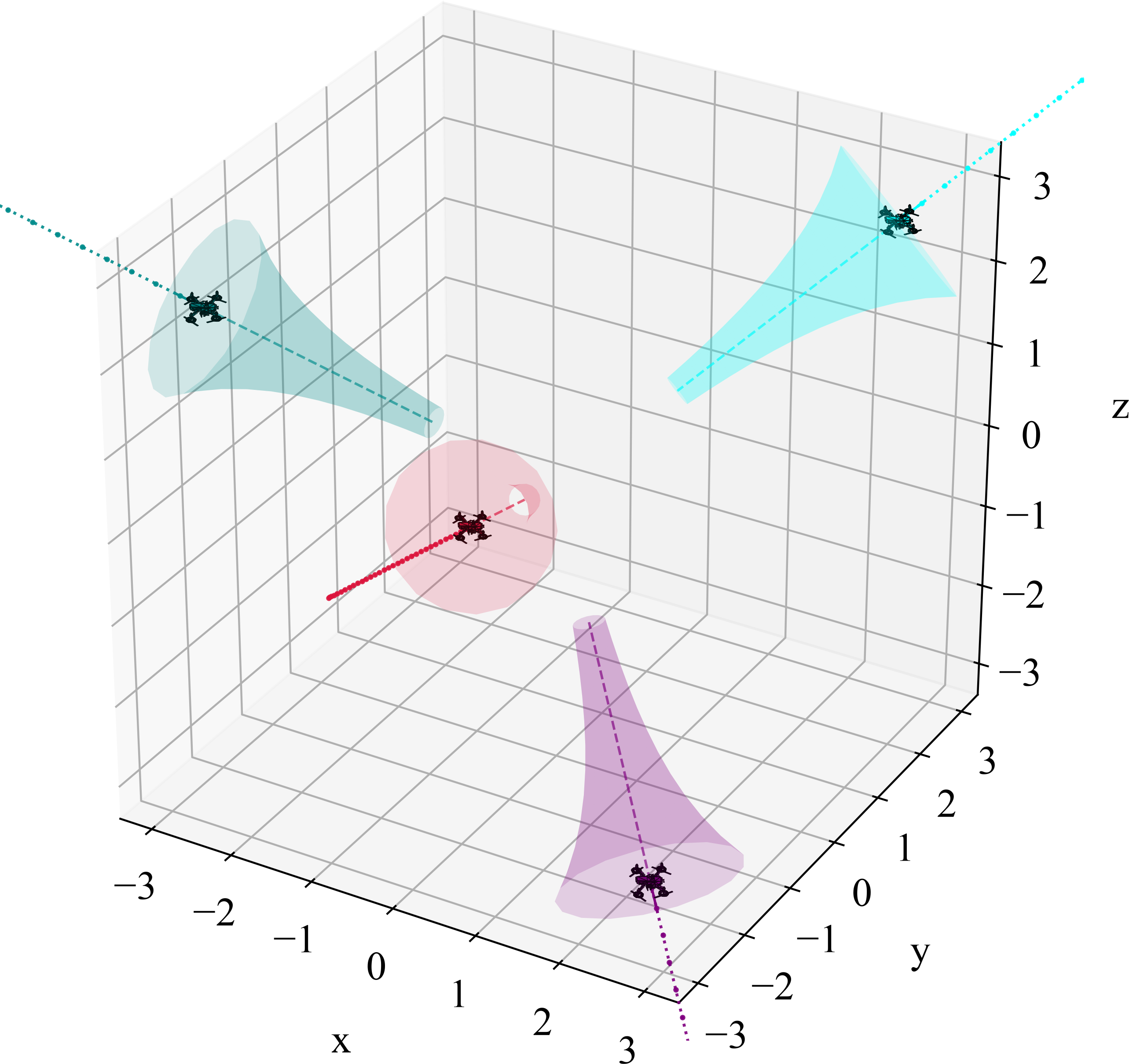}
    \caption{Confidence funnels in a four-drone swap (3D, color-coded). Each funnel runs along a drone's predicted flight line (dashed), starting at the initial radius $r_{\max}$ and narrowing to the lower bound $r_{\mathrm{f}}$ (\autoref{def:funnel}). Dotted lines trace the paths already flown. The funnels shown are those predicted by a single drone, but as the simulation shares positions by broadcast, every drone observes the same histories and predicts the same funnels.}
    \label{fig:funnel}
\end{figure}

\subsection{Conflict Detection}\label{ssec:detection}

Given the two linear predictions, the conflict test reduces to a single closed-form evaluation per neighbor. Let $\Delta\mathbf{p}_{ij} = \mathbf{p}_i - \mathbf{p}_j$ and $\Delta\mathbf{v}_{ij} = \tilde{\mathbf{v}}_i - \tilde{\mathbf{v}}_j$ denote the relative position and velocity. The time of closest approach within the prediction window minimizes $\|\Delta\mathbf{p}_{ij} + \tau\,\Delta\mathbf{v}_{ij}\|$ and is obtained in closed form by clamping the unconstrained minimizer to $[0, t_{\max}]$,
\begin{align}\label{eq:tca}
    \tau^{*}_{ij} =
    \begin{cases}
        \operatorname{clamp}\!\Big(\!-\dfrac{\Delta\mathbf{p}_{ij}^{\top}\Delta\mathbf{v}_{ij}}{\|\Delta\mathbf{v}_{ij}\|^{2}},\, 0,\, t_{\max}\Big), & \|\Delta\mathbf{v}_{ij}\| > \epsilon_0, \\[4pt]
        0, & \text{otherwise,}
    \end{cases}
\end{align}
where $\epsilon_0 > 0$ guards against the singularity of parallel motion. The closed form results from the classical closest-point-of-approach test of the velocity-obstacle literature~\cite{fiorini1998velocity}. The corresponding miss distance, or gap, is
\begin{align}\label{eq:gap}
    g_{ij} = \big\|\Delta\mathbf{p}_{ij} + \tau^{*}_{ij}\,\Delta\mathbf{v}_{ij}\big\|.
\end{align}

\begin{definition}[Predicted conflict]\label{def:conflict}
Neighbor $j$ is in predicted conflict with drone $i$ when the gap at closest approach lies within the funnel evaluated at that instant,
\begin{align}\label{eq:conflict}
    g_{ij} \leq \psi_{ij}\big(\tau^{*}_{ij}\big).
\end{align}
The set of conflicting neighbors of drone $i$ is $\mathcal{C}_i = \{\, j \in \mathcal{N}_i : g_{ij} \leq \psi_{ij}(\tau^{*}_{ij}) \,\}$.
\end{definition}

The test in~\eqref{eq:conflict} is permissive for imminent encounters and increasingly demanding for distant ones, exactly matching the confidence of the linear prediction. This behavior admits a precise statement.

\begin{lemma}[Funnel envelope]\label{lem:funnel_envelope}
For all $\tau \in [0, t_{\max}]$, the funnel radius~\eqref{eq:funnel} is nonincreasing in $\tau$, strictly decreasing whenever $r_{\mathrm{f}} < r_{\max}$, and bounded by
\begin{align}
    r_{\mathrm{f}} + (r_{\max} - r_{\mathrm{f}})\, e^{-1/c_{\mathrm{d}}} \;\leq\; \psi_{ij}(\tau) \;\leq\; r_{\max}.
\end{align}
Consequently, the conflict test~\eqref{eq:conflict} flags every encounter whose predicted gap lies within the neighbor's current adaptive radius, $g_{ij} \leq r_{\mathrm{f}} = r_j(\hat{v}_j)$, at any prediction time, while an encounter is tested against the widest tolerance $r_{\max}$ only at $\tau^{*}_{ij} = 0$.
\end{lemma}

\begin{proof}
$\sigma$ is increasing in $\tau$ and $r_{\mathrm{f}} = r_j(\hat v_j) \leq r_{\max}$ by~\eqref{eq:adaptive_radius} with $\hat v_j \leq V_{\max}$, so $\psi_{ij}$ is nonincreasing, strictly decreasing for $r_{\mathrm{f}} < r_{\max}$, with $\psi_{ij}(0) = r_{\max}$ and $\psi_{ij}(t_{\max}) = r_{\mathrm{f}} + (r_{\max} - r_{\mathrm{f}})e^{-1/c_{\mathrm{d}}}$. In particular $\psi_{ij}(\tau) \geq r_{\mathrm{f}}$ for all $\tau$, so $g_{ij} \leq r_{\mathrm{f}}$ implies~\eqref{eq:conflict}, and $\psi_{ij}(\tau) = r_{\max}$ only at $\tau = 0$.
\end{proof}

\autoref{lem:funnel_envelope} separates the two regimes the funnel is designed for: a predicted violation of the neighbor's own safety zone is never discounted, no matter how far ahead it occurs, whereas encounters outside that zone are discounted more strongly the farther ahead they are predicted.

\subsection{Conflict-Driven Variable Horizon}\label{ssec:horizon}

We propose to set the prediction horizon based on the time to conflict, so that the planning window covers every predicted conflict but reaches no farther than necessary: the horizon must span the predicted conflict, $H_i \Delta t \geq \tau^{*}_{ij}$, clipped to a fixed admissible band $[H_{\min}, H_{\max}]$. Throughout, \emph{admissible} qualifies the horizon and means membership in this band, while \emph{constraint-admissible} qualifies states that satisfy the airspace and speed limits.

\begin{definition}[Variable-horizon policy]\label{def:horizon}
At decision time $k\Delta t$, drone $i$ selects the prediction horizon
\begin{align}\label{eq:horizon}
    H_i(k) =
    \begin{cases}
        \operatorname{clip}\!\Big(\displaystyle\max_{j \in \mathcal{C}_i}\big\lceil \tau^{*}_{ij}/\Delta t \big\rceil,\ H_{\min},\ H_{\max}\Big), & \mathcal{C}_i \neq \emptyset, \\[6pt]
        H_{\min}, & \mathcal{C}_i = \emptyset,
    \end{cases}
\end{align}
where $0 < H_{\min} \leq H_{\max}$ are fixed integer bounds.
\end{definition}

The choice~\eqref{eq:horizon} is the cheapest admissible policy that covers all predicted conflicts.

\begin{proposition}[Minimal covering horizon]\label{prop:minimal_horizon}
Let $\mathcal{C}_i \neq \emptyset$. The policy~\eqref{eq:horizon} covers every predicted conflict, $H_i(k)\,\Delta t \geq \tau^{*}_{ij}$ for all $j \in \mathcal{C}_i$, and it is minimal: every integer horizon $H \in [H_{\min}, H_{\max}]$ whose window covers all predicted conflicts satisfies $H \geq H_i(k)$. If $\mathcal{C}_i = \emptyset$, the policy returns the smallest admissible horizon $H_{\min}$.
\end{proposition}

\begin{proof}
By~\eqref{eq:tca}, $\tau^{*}_{ij} \leq t_{\max} = H_{\max}\Delta t$, so $\lceil \tau^{*}_{ij}/\Delta t\rceil \leq H_{\max}$ and the upper clip in~\eqref{eq:horizon} is never active. Hence $H_i(k) = \max\{H_{\min},\ \max_{j \in \mathcal{C}_i} \lceil \tau^{*}_{ij}/\Delta t\rceil\} \geq \tau^{*}_{ij}/\Delta t$ for every $j \in \mathcal{C}_i$, which is coverage. Conversely, an integer $H$ with $H\Delta t \geq \tau^{*}_{ij}$ for all $j \in \mathcal{C}_i$ satisfies $H \geq \max_j\lceil \tau^{*}_{ij}/\Delta t\rceil$, and admissibility requires $H \geq H_{\min}$, so $H \geq H_i(k)$.
\end{proof}

Since the per-step computational cost increases superlinearly with the horizon (\autoref{ssec:cost}), \autoref{prop:minimal_horizon} shows that requirement \textbf{(E)} holds by construction rather than by tuning. Since the relative travel to closest approach is $\|\Delta\mathbf{v}_{ij}\|\,\tau^{*}_{ij}$, a conflict that is distant in space corresponds to a large $\tau^{*}_{ij}$ and hence a long horizon and early anticipation, whereas a near conflict yields a short horizon and a fast reaction. The lower clip to $H_{\min}$ is the operative guarantee exploited in \autoref{sec:theory}.

\subsection{Integration into the Distributed MPC}\label{ssec:integration}

The horizon~\eqref{eq:horizon} is the only quantity the scheme exposes to the controller. At each step, drone $i$ recomputes $H_i(k)$ from the current history and resizes its local solver and ADMM state to $H_i(k)$ before the optimization. Warm starts from the previous step are padded or truncated to the new length to preserve solver efficiency across horizon changes. This leaves the cost, the constraints, and the iteration scheme of the DMPC unchanged: a solver invoked with a constant horizon recovers the fixed-horizon controller exactly.

One control step proceeds as in Algorithm~\ref{alg:step}. It mirrors the block structure of \autoref{fig:overview}.

\begin{algorithm}[t]
\caption{One control step of drone $i$ at time $k\Delta t$}
\label{alg:step}
\begin{algorithmic}[1]
\REQUIRE histories $\{\mathbf{p}_j(k{-}L{+}1{:}k)\}_{j \in \mathcal{N}_i}$, goal $\bar{\mathbf{p}}_i$, previous solution
\STATE fit the ego line $\hat{\mathbf{p}}_i(\cdot)$~\eqref{eq:linear_pred}--\eqref{eq:speed_floor}; $\mathcal{C}_i \leftarrow \emptyset$
\FOR{each neighbor $j \in \mathcal{N}_i$}
    \STATE fit the flight line $\hat{\mathbf{p}}_j(\cdot)$ from the history~\eqref{eq:linear_pred}--\eqref{eq:speed_floor}
    \STATE compute $\tau^{*}_{ij}$ and $g_{ij}$~\eqref{eq:tca}--\eqref{eq:gap}
    \IF{$g_{ij} \leq \psi_{ij}(\tau^{*}_{ij})$~\eqref{eq:funnel},~\eqref{eq:conflict}}
        \STATE $\mathcal{C}_i \leftarrow \mathcal{C}_i \cup \{j\}$
    \ENDIF
\ENDFOR
\STATE $H_i(k) \leftarrow$ variable-horizon policy~\eqref{eq:horizon} from $\mathcal{C}_i$
\STATE resize solver and ADMM state to $H_i(k)$; pad or trim the warm start
\STATE solve the local MPC of \autoref{ssec:preliminaries}; apply the first input $\mathbf{u}_i(k)$
\end{algorithmic}
\end{algorithm}

Throughout the analysis we require the lower bound to meet the feasibility condition of the underlying controller, which is the only condition the variable horizon has to respect.

\begin{assumption}[Feasible minimum horizon]\label{ass:hmin}
The lower bound of \autoref{def:horizon} satisfies $H_{\min} \geq H_{\min}^{\mathrm{feas}}$, where
\begin{align}\label{eq:hmin_feas}
    H_{\min}^{\mathrm{feas}} = \left\lceil \frac{1}{\Delta t}\sqrt{\frac{2 r_{\min}\alpha}{U_{\max}}}\,\right\rceil
\end{align}
is the minimum prediction horizon for per-step solver feasibility of the underlying MPC at the operating density~\cite{adaptive_dmpc_2026}.
\end{assumption}

Note that \autoref{ass:hmin} constrains only the constant $H_{\min}$, not the time-varying horizon, and is verified once from the airspace and dynamics parameters. Hence, every horizon produced by~\eqref{eq:horizon} lies in $[H_{\min}^{\mathrm{feas}}, H_{\max}]$, which, as \autoref{sec:theory} shows, is sufficient to preserve the feasibility and stability guarantees of the underlying controller.

\section{Feasibility and Stability of the Variable-Horizon DMPC}\label{sec:theory}

The fixed-horizon adaptive DMPC of~\cite{adaptive_dmpc_2026} is proven feasible and stable for any \emph{constant} admissible horizon under per-step solver feasibility, a property we make explicit in \autoref{ssec:var_guarantees}. These results do not simply transfer to a horizon that changes from step to step. Standard MPC arguments certify stability with the optimal value function, which depends on the horizon length, so a switch could break the descent between two steps, and recursive feasibility via the shifted previous solution presumes that consecutive problems share their length. The goal of this section is to prove feasibility and asymptotic stability for the conflict-driven horizon of \autoref{sec:framework}, which is drawn at every step from the admissible band $[H_{\min}, H_{\max}]$ as in \autoref{def:horizon}. Two steps make this possible. We work with a certificate that is a Lyapunov function of the physical state alone and therefore does not change when the horizon does, and we prove that the single horizon-dependent quantity in the analysis, a finite-horizon residual, is uniformly bounded over the band. We first recall the fixed-horizon guarantees (\autoref{ssec:recalled}), cast their stability argument as a single perturbation bound (\autoref{ssec:perturbation}), and fix the horizon bounds (\autoref{ssec:bounds}). We then establish feasibility and asymptotic stability under the variable horizon (\autoref{ssec:var_guarantees}), before isolating the principle that safety is decoupled from prediction accuracy (\autoref{ssec:decoupling}). \autoref{sec:extended_models} lifts every result to quadrotor dynamics.

\subsection{Fixed-Horizon DMPC Guarantees}\label{ssec:recalled}

We summarize the results of~\cite{adaptive_dmpc_2026} on which the analysis builds. To this end, we recall the necessary propositions, for their proofs see~\cite{adaptive_dmpc_2026}. The geometric capacity of the airspace under the adaptive radius~\eqref{eq:adaptive_radius} follows from the three-dimensional sphere-packing bound~\cite{drones10020139}: with $\delta_3 = \pi/\sqrt{18} \approx 0.74$ the optimal packing density, the maximum drone count at radius $r$ is $N_{\mathrm{c}}(r) = \lfloor \delta_3\, \mathrm{vol}(\Omega) / (\tfrac{4}{3}\pi r^3)\rfloor$, and the critical adaptive density is
\begin{align}\label{eq:n_crit_var}
    N_{\mathrm{c}}^{\mathrm{v}} = N_{\mathrm{c}}(r_{\min} + \varepsilon)
    = \left\lfloor \delta_3 \cdot \frac{\mathrm{vol}(\Omega)}{\tfrac{4}{3}\pi (r_{\min} + \varepsilon)^3} \right\rfloor,
\end{align}
with $\varepsilon > 0$ arbitrarily small. The horizon required for the solver to plan beyond a neighbor's safety zone is likewise bounded below~\cite{adaptive_dmpc_2026}: under $N \leq N_{\mathrm{c}}^{\mathrm{v}}$, a feasible trajectory exists only if $H \geq H_{\min}^{\mathrm{feas}}$, the bound~\eqref{eq:hmin_feas} already stated in \autoref{ass:hmin}, computable in closed form from the base radius, the adaptation parameter, the acceleration limit, and the sampling time. Stability rests on a Lyapunov function defined on the physical state, and therefore one that does not change with the prediction horizon,
\begin{align}\label{eq:lyapunov}
    V(\mathbf{x}) = \sum_{i=1}^{N} \Big[ \|\mathbf{p}_i - \bar{\mathbf{p}}_i\|^2 + \gamma \|\mathbf{v}_i\|^2 + \delta \big(r_i^2(t) - r_{\min}^2\big) \Big],
\end{align}
with $\gamma, \delta > 0$ and equilibrium value $V(\mathbf{x}^\star) = 0$ at $\mathbf{x}^\star = \{(\bar{\mathbf{p}}_i, \mathbf{0})\}_{i=1}^N$. Near the goal, the separation constraints are inactive and the adaptive radii contract to $r_{\min}$, so each drone's local subproblem reduces to the unconstrained linear--quadratic regulator (LQR). This yields the nominal feedback that drives $V$ downward. Throughout, the Lyapunov analysis is carried out in continuous time for the sample-and-hold closed loop, following~\cite{adaptive_dmpc_2026}: the zero-order-hold mismatch between sampling instants is a bounded feedback perturbation of order $\mathcal{O}(\Delta t)$ that vanishes near equilibrium and is absorbed by the perturbation bound of \autoref{ssec:perturbation}.

\begin{proposition}[Nominal Lyapunov Descent~{\cite{adaptive_dmpc_2026}}]\label{prop:nominal_descent}
Suppose $\alpha < \alpha_{\mathrm{c}} = 2 r_{\min} U_{\max}/V_{\max}^2$ and $N < N_{\mathrm{c}}^{\mathrm{v}}$. Under the LQR feedback $\mathbf{u}_i^{\mathrm{lqr}} = -\mathbf{K}_p(\mathbf{p}_i - \bar{\mathbf{p}}_i) - \mathbf{K}_v \mathbf{v}_i$ with $\mathbf{K}_p, \mathbf{K}_v > 0$, the closed loop satisfies
\begin{align}\label{eq:nominal_descent}
    \dot V \leq -c_p \sum_{i=1}^N \|\mathbf{p}_i - \bar{\mathbf{p}}_i\|^2 - c_v \sum_{i=1}^N \|\mathbf{v}_i\|^2 \leq -\lambda_{\min} V,
\end{align}
with $c_p, c_v, \lambda_{\min} > 0$ independent of the prediction horizon.
\end{proposition}

The finite-horizon DMPC realizes this feedback only approximately, and only up to the coupling between neighboring subproblems. Both deviations are collected in a single perturbation term.

\begin{proposition}[Perturbed LQR Approximation~{\cite{adaptive_dmpc_2026}}]\label{prop:perturbed_lqr}
At the fixed point of the distributed iteration with horizon $H$, each local DMPC solution takes the form
\begin{align}\label{eq:perturbed_lqr}
    \mathbf{u}_i^* = \mathbf{u}_i^{\mathrm{lqr}} + \mathbf{e}_i, \qquad \mathbf{e}_i = \boldsymbol{\rho}_i(H) + \eta_i,
\end{align}
where $\boldsymbol{\rho}_i(H)$ is the finite-horizon approximation residual and $\eta_i$ is the coupling perturbation with $\|\eta_i\| \leq C\beta$, $C > 0$ depending on the constraint geometry and $\beta < 1$ the contraction rate of the distributed update map. Both terms vanish as the system approaches equilibrium.
\end{proposition}

Over the constraint-admissible set $\Omega \times \{\|\mathbf{v}_i\| \leq V_{\max}\}$ the residual is moreover uniformly bounded, $\|\boldsymbol{\rho}_i(H)\| \leq \bar\rho < \infty$ for every $H$, since every feasible input satisfies $\|\mathbf{u}_i^*\| \leq U_{\max}$ and the LQR feedback is bounded on this compact set.

\subsection{A Perturbation Bound for the Adaptive DMPC}\label{ssec:perturbation}

The stability of the DMPC, and of all its extensions in this paper, follows from a single observation: the nominal descent of \autoref{prop:nominal_descent} tolerates any bounded perturbation of the feedback, in the sense of input-to-state stability (ISS), with an ultimate bound proportional to the perturbation size.

\begin{lemma}[ISS under Bounded Feedback Perturbation]\label{lem:iss_perturbation}
Let $\alpha < \alpha_{\mathrm{c}}$ and $N < N_{\mathrm{c}}^{\mathrm{v}}$, and suppose each drone applies $\mathbf{u}_i = \mathbf{u}_i^{\mathrm{lqr}} + \mathbf{e}_i$ with $\sup_i \|\mathbf{e}_i\| \leq \bar e$. Then along the closed loop
\begin{align}
    \dot V &\leq -c_p \sum_{i=1}^N \|\mathbf{p}_i - \bar{\mathbf{p}}_i\|^2 - \sum_{i=1}^N \|\mathbf{v}_i\|\big(c_v \|\mathbf{v}_i\| - 2\gamma_{\mathrm{eff}}\, \bar e\big) \label{eq:iss_perturb} \\
    &\leq -c_p \sum_{i=1}^N \|\mathbf{p}_i - \bar{\mathbf{p}}_i\|^2 + \frac{N (\gamma_{\mathrm{eff}}\, \bar e)^2}{c_v}, \label{eq:iss_quad}
\end{align}
with $\gamma_{\mathrm{eff}} = \gamma + \delta\alpha r_{\max}/U_{\max}$ and $r_{\max} = r_{\min} + \alpha V_{\max}^2/(2U_{\max})$. Consequently $\dot V < 0$ outside the set $\{\max_i\|\mathbf{v}_i\| \leq 2\gamma_{\mathrm{eff}}\bar e / c_v\}$, so the closed loop is practically asymptotically stable with ultimate bound $\mathcal{O}(\bar e)$. If $\bar e \to 0$ as the state approaches equilibrium, the equilibrium is asymptotically stable.
\end{lemma}

\begin{proof}
Substituting $\mathbf{u}_i = \mathbf{u}_i^{\mathrm{lqr}} + \mathbf{e}_i$ into $\dot V$ and using the differentiation of~\eqref{eq:lyapunov} from~\cite{adaptive_dmpc_2026}, the nominal terms give the descent~\eqref{eq:nominal_descent}, while each $\mathbf{e}_i$ contributes $2(\gamma + \delta\alpha r_i/U_{\max})\,\mathbf{v}_i^\top \mathbf{e}_i$. Bounding $r_i \leq r_{\max}$ and applying the Cauchy--Schwarz inequality with $\|\mathbf{e}_i\| \leq \bar e$ yields the perturbation term $2\gamma_{\mathrm{eff}}\bar e \sum_i \|\mathbf{v}_i\|$, hence~\eqref{eq:iss_perturb}. Each velocity summand is negative once $\|\mathbf{v}_i\| > 2\gamma_{\mathrm{eff}}\bar e/c_v$, and maximizing it over $\|\mathbf{v}_i\|$ gives~\eqref{eq:iss_quad}. Together with the negative-definite position term this gives $\dot V < 0$ outside an $\mathcal{O}(\bar e)$ neighborhood of the equilibrium. When $\bar e \to 0$ along trajectories approaching the goal, the radius of this neighborhood shrinks to zero and asymptotic stability follows.
\end{proof}

Applying \autoref{lem:iss_perturbation} to the genuine DMPC feedback~\eqref{eq:perturbed_lqr} recovers the fixed-horizon stability result of~\cite{adaptive_dmpc_2026} as its first special case.

\begin{corollary}[Stability of the Adaptive DMPC]\label{cor:dmpc_stability}
Under $\alpha < \alpha_{\mathrm{c}}$, $N < N_{\mathrm{c}}^{\mathrm{v}}$, $H \geq H_{\min}^{\mathrm{feas}}$, and the coupling condition
\begin{align}
    \beta < \beta_{\mathrm{c}} = \frac{c_v}{\gamma_{\mathrm{eff}}\, C},
\end{align}
the fixed-horizon adaptive DMPC with asynchronous Gauss--Seidel updates is asymptotically stable about $\{(\bar{\mathbf{p}}_i, \mathbf{0})\}_{i=1}^N$.
\end{corollary}

\begin{proof}
By \autoref{prop:perturbed_lqr} the realized feedback is $\mathbf{u}_i^{\mathrm{lqr}} + \mathbf{e}_i$ with $\mathbf{e}_i = \boldsymbol{\rho}_i(H) + \eta_i$, so $\bar e \leq \sup_i\|\boldsymbol{\rho}_i(H)\| + C\beta$. \autoref{lem:iss_perturbation} gives practical stability with ultimate bound $\mathcal{O}(\bar e)$. Since $\boldsymbol{\rho}_i(H)$ and $\eta_i$ both vanish as the system approaches equilibrium, $\bar e \to 0$ and the asymptotic stability follows. The condition $\beta < \beta_{\mathrm{c}}$ is exactly $2\gamma_{\mathrm{eff}} C\beta < 2 c_v$, ensuring the coupling contribution is dominated by the velocity dissipation.
\end{proof}

\subsection{Choice of the Horizon Bounds}\label{ssec:bounds}

The variable horizon is confined to the band $[H_{\min}, H_{\max}]$ fixed once before deployment, and only the lower bound enters the guarantees. The minimum is chosen to meet the feasibility requirement~\eqref{eq:hmin_feas}, $H_{\min} \geq H_{\min}^{\mathrm{feas}}$. The maximum $H_{\max}$ sets the prediction window $t_{\max} = H_{\max}\Delta t$ over which conflicts are anticipated and caps the worst-case per-step computation. It affects neither feasibility nor stability and need only satisfy $H_{\min} \leq H_{\max} < \infty$. A conflict beyond $t_{\max}$ is simply detected at a later step once it enters the window, while the lower bound $H_{\min}$ keeps every window long enough for a braking maneuver~\eqref{eq:hmin_feas}. A larger $H_{\max}$ yields earlier anticipation at higher peak cost. We therefore treat both bounds as fixed design constants meeting \autoref{ass:hmin} and carry no further assumption on the horizon into the analysis.

\subsection{Feasibility and Stability under the Variable Horizon}\label{ssec:var_guarantees}

The variable horizon threatens feasibility at exactly one point. When the horizon grows between steps, the shifted previous solution falls short of the new subproblem's length, and since the framework carries no terminal set, no terminal controller is available to extend it in the standard way. This subsection establishes that feasibility survives every admissible horizon change nonetheless. From~\cite{adaptive_dmpc_2026} we import two facts. A collision-free configuration exists at density $N \leq N_{\mathrm{c}}^{\mathrm{v}}$ by the packing bound~\cite{drones10020139}, and per-step solver feasibility holds for every constant horizon $H \geq H_{\min}^{\mathrm{feas}}$~\eqref{eq:hmin_feas}, which we state as \autoref{ass:fixed_feas}. What we prove is that the time-varying horizon preserves this property, constructively via the truncated shift when the horizon shrinks (\autoref{lem:braking}) and by \autoref{ass:fixed_feas} when it grows.

\begin{assumption}[Per-step fixed-horizon feasibility]\label{ass:fixed_feas}
At density $N < N_{\mathrm{c}}^{\mathrm{v}}$, each drone's fixed-horizon subproblem is feasible from every constraint-admissible state reached in closed loop, for every constant horizon $H \geq H_{\min}^{\mathrm{feas}}$~\cite{adaptive_dmpc_2026}.
\end{assumption}

\begin{lemma}[Admissible-Horizon Feasibility]\label{lem:braking}
Let $N < N_{\mathrm{c}}^{\mathrm{v}}$, \autoref{ass:hmin}, and \autoref{ass:fixed_feas} hold. Then each drone's subproblem is feasible for every $H \in [H_{\min}, H_{\max}]$, and the one-step-shifted previous solution, truncated to any $H' \in [H_{\min}, H-1]$, remains a feasible point of the shorter subproblem.
\end{lemma}

\begin{proof}
Every $H \in [H_{\min}, H_{\max}]$ satisfies $H \geq H_{\min}^{\mathrm{feas}}$ by \autoref{ass:hmin}, so feasibility at each step is \autoref{ass:fixed_feas}. The constraints are stage-wise, so the length-$H'$ prefix of the shifted feasible trajectory satisfies all constraints of the horizon-$H'$ subproblem.
\end{proof}

The Lyapunov function~\eqref{eq:lyapunov} depends only on the physical state, and the perturbation bound of \autoref{lem:iss_perturbation} depends on the horizon only through the residual $\boldsymbol{\rho}_i(H)$ of \autoref{prop:perturbed_lqr}. Since that residual is bounded by $\bar\rho$ uniformly in the horizon over the constraint-admissible set, its worst case over the band is $\bar\rho$, uniformly in time. This is why the time-varying horizon preserves the guarantees.

\begin{theorem}[Feasibility and Stability under the Variable Horizon]\label{thm:var_guarantees}
Let $N < N_{\mathrm{c}}^{\mathrm{v}}$ and let \autoref{ass:hmin} and \autoref{ass:fixed_feas} hold. Then for every horizon sequence $\{H_i(k)\}_{k\geq 0} \subset [H_{\min}, H_{\max}]$ produced by \autoref{def:horizon}, the distributed adaptive MPC is feasible at every step. If in addition $\alpha < \alpha_{\mathrm{c}}$ and $\beta < \beta_{\mathrm{c}}$ as in \autoref{cor:dmpc_stability}, the closed loop with asynchronous Gauss--Seidel updates is asymptotically stable about $\{(\bar{\mathbf{p}}_i, \mathbf{0})\}_{i=1}^N$.
\end{theorem}

\begin{proof}
\emph{Feasibility:} Feasibility at every step follows for any horizon sequence, using only that every selected horizon lies in the band $[H_{\min}, H_{\max}]$. Each step is thus a fixed-horizon subproblem at $H_i(k)$, feasible at $N < N_{\mathrm{c}}^{\mathrm{v}}$ by \autoref{lem:braking}.

Horizon changes between steps affect only the warm start. Shifting drone $i$'s step-$k$ solution forward gives a feasible length-$(H_i(k){-}1)$ trajectory (\autoref{lem:braking}), truncated when $H_i(k{+}1)$ is shorter and padded when it is longer. The longer subproblem is feasible by \autoref{lem:braking}. The padded warm start only initializes the solver and carries no feasibility burden. This is the warm-start handling of \autoref{ssec:integration}. The argument applies to each drone individually, and the Gauss--Seidel contraction of \autoref{cor:dmpc_stability} is a per-step property, so convergence to a feasible fixed point is unaffected.

\emph{Stability:} On each sampling interval the active horizon is a fixed value $H_i(k) \in [H_{\min}, H_{\max}]$, so the realized feedback is $\mathbf{u}_i^{\mathrm{lqr}} + \mathbf{e}_i$ with $\mathbf{e}_i = \boldsymbol{\rho}_i(H_i(k)) + \eta_i$ as in~\eqref{eq:perturbed_lqr}. Because the residual bound $\bar\rho$ is uniform in the horizon, $\|\boldsymbol{\rho}_i(H_i(k))\| \leq \bar\rho$ for every admissible horizon, uniformly over the constraint-admissible set, so $\bar e \leq \bar\rho + C\beta$ uniformly in $i$ and $k$. This bound is independent of which horizon the policy selects, and \autoref{lem:iss_perturbation} therefore applies with a single $\bar e$ on every interval, using the common, horizon-free Lyapunov function~\eqref{eq:lyapunov}. In particular, because $V$ does not depend on the horizon, it is continuous across the switching instants, and a switch merely moves $\boldsymbol{\rho}_i$ within the fixed bound. No dwell-time condition is therefore needed. Far from equilibrium, feasibility at every step holds as established above, and the speed and airspace constraints confine each trajectory to the compact constraint-admissible set, on which $\bar e$ is finite and uniform in time. By~\eqref{eq:iss_quad}, $V$ decreases strictly until the trajectory enters an $\mathcal{O}(\bar e)$ neighborhood of the goal configuration. There $\boldsymbol{\rho}_i$ and $\eta_i$ vanish as the state approaches equilibrium (\autoref{prop:perturbed_lqr}), so $\bar e$ contracts and the radius of the neighborhood shrinks to zero. Asymptotic stability follows as in \autoref{cor:dmpc_stability}.
\end{proof}

\begin{remark}[Variable horizon and the coupling condition]\label{rem:beta_relax}
We expect the coupling margin to improve rather than degrade under the variable horizon. Enlarging the horizon precisely when a conflict is anticipated lets the controller resolve it earlier and with weaker constraint activation, which in practice weakens the coupling between neighboring subproblems and thus reduces the contraction rate $\beta$ of \autoref{prop:perturbed_lqr}. We do not formalize this link, and the guarantees of \autoref{thm:var_guarantees} do not rely on it.
\end{remark}

\subsection{Decoupling of Safety from Prediction Accuracy}\label{ssec:decoupling}

A central feature of the construction is that \autoref{thm:var_guarantees} invokes the predictor of \autoref{ssec:prediction}--\autoref{ssec:detection} only through the lower bound \autoref{ass:hmin}. None of the proofs assumes that the linear prediction is correct, that the funnel captures the true conflict, or that the conflict set $\mathcal{C}_i$ is complete. Whatever horizon \autoref{def:horizon} returns lies in $[H_{\min}, H_{\max}]$, and that membership alone carries the guarantees.

\begin{corollary}[Robustness to Misprediction]\label{cor:robustness}
The guarantees of \autoref{thm:var_guarantees} hold for \emph{any} horizon policy with range in $[H_{\min}, H_{\max}]$ under \autoref{ass:hmin} and \autoref{ass:fixed_feas}, irrespective of the prediction model, the history length $L$, or the funnel parameters $c_{\mathrm{d}}, \nu$.
\end{corollary}

This robustness concerns feasibility and stability, not the separation margin. A predictor that never reports a conflict keeps the horizon at $H_{\min}$ and retains the guarantees under \autoref{ass:hmin} and \autoref{ass:fixed_feas}, but the drones then pass closer to the contact boundary (\autoref{ssec:separation}). Prediction quality thus governs \emph{efficiency} rather than safety: it selects \emph{which} horizon in the certified band is used, trading computation against anticipation. The link to anticipation is made precise by the conflict-resolution budget of~\cite{drones10020139}: with the per-lane headway $h_{\min} = 2r/V_{\max} + V_{\max}/U_{\max}$, a horizon $H$ resolves up to $\kappa(H) = \lfloor H\Delta t / h_{\min}\rfloor$ sequential conflicts. A conflict predicted at time to conflict $\tau^*_{ij}$ raises the horizon to at least $\lceil \tau^*_{ij}/\Delta t\rceil$ by~\eqref{eq:horizon}, enlarging $\kappa$ exactly where a conflict must be scheduled, while a clear neighborhood keeps $H = H_{\min}$ and the per-step cost minimal. A misprediction can only waste budget, through an unnecessarily long or short horizon within the band. It never violates safety.

\section{Extension to Quadrotor Dynamics}\label{sec:extended_models}

The analysis so far uses the double-integrator model~\eqref{eq:dynamics}. We now extend all results to two higher-fidelity quadrotor models: the linearization about hover (\autoref{ssec:lin_quad}) and the full nonlinear rigid body (\autoref{ssec:nonlin_quad}). In both cases a cascaded inner--outer-loop controller reduces the translational dynamics to a perturbed double integrator, and each model instantiates the bound $\bar e$ of \autoref{lem:iss_perturbation} with a concrete residual. The extension is therefore immediate, and because \autoref{lem:iss_perturbation} is horizon-uniform, it carries the variable horizon along unchanged. The two models are representatives of a class rather than endpoints: the analysis covers any copter model whose additional effects perturb the translational loop by a bounded residual, a point we make precise at the end of \autoref{ssec:nonlin_quad}.

\subsection{Linearized Quadrotor with Cascaded Control}\label{ssec:lin_quad}

A quadrotor has state $\mathbf{x}_{\mathrm{quad}} = (\mathbf{p}^\top, \mathbf{v}^\top, \boldsymbol{\Theta}^\top, \boldsymbol{\omega}^\top)^\top \in \mathbb{R}^{12}$ with Euler angles $\boldsymbol{\Theta} = (\phi, \theta, \psi)^\top$ and body rates $\boldsymbol{\omega}$. Linearizing the rigid-body equations about hover ($\mathbf{v} = 0$, $\boldsymbol{\Theta} = 0$, $\boldsymbol{\omega} = 0$, $T = mg$) yields the linear time-invariant system
\begin{align}\label{eq:lin_quad}
    \dot{\mathbf{x}}_{\mathrm{quad}} = \mathbf{A}\,\mathbf{x}_{\mathrm{quad}} + \mathbf{B}\,\mathbf{u}_{\mathrm{quad}},
\end{align}
with input $\mathbf{u}_{\mathrm{quad}} = (\delta T, \tau_\phi, \tau_\theta, \tau_\psi)^\top$ and the well-known decoupled channel structure~\cite{mahony2012multirotor}
\begin{align}
    \mathbf{A} = \begin{pmatrix}
        \mathbf{0}_{3} & \mathbf{I}_{3} & \mathbf{0}_{3} & \mathbf{0}_{3} \\
        \mathbf{0}_{3} & \mathbf{0}_{3} & \mathbf{G} & \mathbf{0}_{3} \\
        \mathbf{0}_{3} & \mathbf{0}_{3} & \mathbf{0}_{3} & \mathbf{I}_{3} \\
        \mathbf{0}_{3} & \mathbf{0}_{3} & \mathbf{0}_{3} & \mathbf{0}_{3}
    \end{pmatrix}, \quad
    \mathbf{B} = \begin{pmatrix}
        \mathbf{0}_{3\times4} \\
        \mathbf{B}_v \\
        \mathbf{0}_{3\times4} \\
        \mathbf{B}_\omega
    \end{pmatrix},
\end{align}
where
\begin{align}
    \mathbf{G} &= \begin{pmatrix}
        0 & g & 0 \\
        -g & 0 & 0 \\
        0 & 0 & 0
    \end{pmatrix}, \quad
    \mathbf{B}_v = \begin{pmatrix} 0 \\ 0 \\ 1/m \end{pmatrix} \mathbf{e}_1^\top, \nonumber \\
    \mathbf{B}_\omega &= \mathrm{diag}\!\Big(\tfrac{1}{I_{xx}}, \tfrac{1}{I_{yy}}, \tfrac{1}{I_{zz}}\Big) \begin{pmatrix} \mathbf{0}_{3\times1} & \mathbf{I}_3 \end{pmatrix},
\end{align}
with gravitational acceleration $g$, mass $m$, and principal moments of inertia $I_{xx}, I_{yy}, I_{zz}$. The matrix $\mathbf{G}$ captures the gravity coupling: a small pitch $\theta$ produces a horizontal acceleration $g\theta$ and a small roll $\phi$ produces $-g\phi$. The horizontal channels $(x, \dot x, \theta, q)$ and $(y, \dot y, \phi, p)$ each have relative degree four from the torque inputs, and the vertical channel $(z, \dot z)$ relative degree two from the thrust input, as derived for an experimental quadrotor in~\cite{2019quadrotor}, so the position channels are controllable and feedback-linearizable~\cite{khalil_nonlinear_2002}.

Standard quadrotor practice~\cite{baca2018mpc, 2019quadrotor} uses a cascaded architecture: a fast inner-loop attitude controller of bandwidth $\omega_{\mathrm{i}} \gg 1/\Delta t$ tracks the commanded attitude, $\boldsymbol{\Theta}(t) \approx \boldsymbol{\Theta}_{\mathrm{cmd}}(t)$ with time constant $\tau_{\mathrm{i}} \ll \Delta t$, while the outer loop is the position DMPC. Under time-scale separation~\cite{khalil_nonlinear_2002}, the translational dynamics reduce to
\begin{align}\label{eq:outer_loop}
    \ddot{\mathbf{p}}_i = \mathbf{u}_i + \boldsymbol{\xi}_i(t), \qquad \|\boldsymbol{\xi}_i(t)\| \leq \bar\xi = \mathcal{O}(\tau_{\mathrm{i}}/\Delta t) \ll 1,
\end{align}
that is, the double integrator~\eqref{eq:dynamics} perturbed by the bounded inner-loop residual $\boldsymbol{\xi}_i$. We write $\tilde U_{\max} = U_{\max} - \bar\xi$ for the effective acceleration available to the outer loop.

Feasibility carries over with $U_{\max}$ replaced by $\tilde U_{\max}$: the density bound $N_{\mathrm c}^{\mathrm v}$~\eqref{eq:n_crit_var} is purely geometric and unchanged, the minimum horizon becomes $\tilde H_{\min}^{\mathrm{feas}} = \lceil \tfrac{1}{\Delta t}\sqrt{2 r_{\min}\alpha/\tilde U_{\max}}\,\rceil$, i.e., the bound~\eqref{eq:hmin_feas} with $\tilde U_{\max}$, and the admissible-horizon feasibility of \autoref{lem:braking} carries over under the reduced acceleration. Stability follows directly from the perturbation lemma.

\begin{theorem}[DMPC Stability under Linearized Quadrotor Dynamics]\label{thm:quad_lin}
Consider the adaptive DMPC with the linearized quadrotor dynamics~\eqref{eq:lin_quad} and cascaded control~\eqref{eq:outer_loop}. Under $\alpha < \tilde\alpha_{\mathrm c} = 2 r_{\min}\tilde U_{\max}/V_{\max}^2$, $N < N_{\mathrm c}^{\mathrm v}$, $H \geq \tilde H_{\min}^{\mathrm{feas}}$, and the combined perturbation condition
\begin{align}\label{eq:quad_lin_cond}
    C\beta + \bar\xi < \frac{c_v}{2\gamma_{\mathrm{eff}}},
\end{align}
the closed loop is asymptotically stable about $\{(\bar{\mathbf{p}}_i, \mathbf{0})\}_{i=1}^N$.
\end{theorem}

\begin{proof}
With the cascaded controller the realized feedback is $\mathbf{u}_i^{\mathrm{lqr}} + \mathbf{e}_i$, where now $\mathbf{e}_i = \boldsymbol{\rho}_i(H) + \eta_i + \boldsymbol{\xi}_i$ collects the finite-horizon residual, the distributed coupling, and the inner-loop residual~\eqref{eq:outer_loop}. By \autoref{prop:perturbed_lqr} and~\eqref{eq:outer_loop}, $\bar e \leq \bar\rho + C\beta + \bar\xi$. \autoref{lem:iss_perturbation}, with $U_{\max}$ replaced by $\tilde U_{\max}$ throughout, gives~\eqref{eq:iss_quad}, hence $\dot V < 0$ outside an $\mathcal{O}(\bar e)$ neighborhood. Condition~\eqref{eq:quad_lin_cond} places that neighborhood inside the region where the velocity dissipation dominates the perturbation. The residuals $\boldsymbol{\rho}_i$ and $\eta_i$ vanish near equilibrium. So does $\boldsymbol{\xi}_i$, since $\|\mathbf{v}_i\| \to 0$ forces $\|\dot{\boldsymbol{\Theta}}_{\mathrm{cmd}}\| \to 0$ and the inner-loop residual is driven by the command rate. The ultimate bound therefore shrinks to zero and the equilibrium is asymptotically stable. The coupling and inner-loop perturbations enter $\bar e$ additively, hence consume the stability margin additively.
\end{proof}

In practice, modern flight controllers run attitude loops at $50$--$200\,$Hz and position MPC at $10$--$50\,$Hz, giving $\tau_{\mathrm{i}}/\Delta t \leq 0.1$ and hence a residual $\bar\xi$ that is a small fraction of $U_{\max}$. Condition~\eqref{eq:quad_lin_cond} is then satisfied by a wide margin, and $\tilde\alpha_{\mathrm c} \approx \alpha_{\mathrm c}$, $\tilde H_{\min}^{\mathrm{feas}} \approx H_{\min}^{\mathrm{feas}}$.

\subsection{Nonlinear Rigid-Body Quadrotor}\label{ssec:nonlin_quad}

The full nonlinear model has translational and rotational dynamics~\cite{mahony2012multirotor, 2019quadrotor}
\begin{align}
    m\,\ddot{\mathbf{p}} &= -mg\,\mathbf{e}_3 + T\,\mathbf{R}\,\mathbf{e}_3, \label{eq:nonlin_trans} \\
    \mathbf{J}\,\dot{\boldsymbol{\omega}} &= -\boldsymbol{\omega} \times (\mathbf{J}\,\boldsymbol{\omega}) + \boldsymbol{\tau}, \label{eq:nonlin_rot}
\end{align}
with $\mathbf{R} \in SO(3)$ the body-to-inertial rotation, $T$ the total thrust, $\mathbf{J}$ the inertia matrix, and $\boldsymbol{\tau}$ the body torque. This adds two complications relative to the double integrator: the translational acceleration depends on $\mathbf{R}$, which evolves on the manifold $SO(3)$, and the input-to-acceleration map is nonlinear in the attitude. Both are handled by establishing ISS of the attitude loop and then bounding the resulting translational perturbation.

\begin{lemma}[ISS of the Attitude Loop]\label{lem:iss_attitude}
Consider the rotational dynamics~\eqref{eq:nonlin_rot} under an attitude controller $\boldsymbol{\tau} = \boldsymbol{\tau}(\boldsymbol{\Theta}, \boldsymbol{\omega}, \boldsymbol{\Theta}_{\mathrm{cmd}})$ that exponentially stabilizes the equilibrium $(\tilde{\boldsymbol{\Theta}}, \boldsymbol{\omega}) = (\mathbf{0}, \mathbf{0})$ of the attitude error dynamics, $\tilde{\boldsymbol{\Theta}} = \boldsymbol{\Theta} - \boldsymbol{\Theta}_{\mathrm{cmd}}$, uniformly in the command, on the operating region $\|\boldsymbol{\Theta}\| \leq \Theta_{\max} < \pi/2$. Then the attitude error is locally input-to-state stable with respect to the command rate: there exist a class-$\mathcal{KL}$ function $\beta_a$ and a class-$\mathcal{K}$ function $\gamma_a$ with
\begin{align}\label{eq:iss_bound}
    \|\boldsymbol{\Theta}(t) - \boldsymbol{\Theta}_{\mathrm{cmd}}(t)\| &\leq \beta_a\big(\|\boldsymbol{\Theta}(0) - \boldsymbol{\Theta}_{\mathrm{cmd}}(0)\|, t\big) \nonumber \\
    &\quad + \gamma_a\Big(\sup_{0 \leq s \leq t} \|\dot{\boldsymbol{\Theta}}_{\mathrm{cmd}}(s)\|\Big).
\end{align}
\end{lemma}

\begin{proof}
On the operating region, an exponentially stable equilibrium under an additive Lipschitz perturbation is ISS with a linear gain~\cite[Lem.~4.6, local version Lem.~4.7]{khalil_nonlinear_2002}. Here the perturbation is $-\dot{\boldsymbol{\Theta}}_{\mathrm{cmd}}$, which the command motion adds to $\dot{\tilde{\boldsymbol{\Theta}}}$, so~\eqref{eq:iss_bound} holds with $\gamma_a$ linear.
\end{proof}
The exponential-stability hypothesis is met near hover by the standard PD~\cite{2019quadrotor}, quaternion-based, and geometric~\cite{lee2010geometric} attitude controllers.

Writing the translational dynamics~\eqref{eq:nonlin_trans} as a perturbed double integrator, $\ddot{\mathbf{p}} = \mathbf{u}_{\mathrm{nom}} + \boldsymbol{\xi}_{\mathrm{nl}}$, the nonlinear residual is
\begin{align}\label{eq:xi_nl_def}
    \boldsymbol{\xi}_{\mathrm{nl}} &= \tfrac{T}{m}\big(\mathbf{R}(\boldsymbol{\Theta})\,\mathbf{e}_3 - \mathbf{R}(\boldsymbol{\Theta}_{\mathrm{cmd}})\,\mathbf{e}_3\big) \nonumber \\
    &\quad + \big(\tfrac{T}{m}\,\mathbf{R}(\boldsymbol{\Theta}_{\mathrm{cmd}})\,\mathbf{e}_3 - g\,\mathbf{e}_3 - \mathbf{u}_{\mathrm{nom}}\big),
\end{align}
whose first term is the attitude tracking error and whose second is the thrust/attitude inversion mismatch. The cascaded controller extracts the command $(T, \boldsymbol{\Theta}_{\mathrm{cmd}})$ from the nominal acceleration by exact thrust-vector inversion, exploiting the differential flatness of the quadrotor~\cite{mellinger2011minimum},
\begin{align}\label{eq:flat_inversion}
    T = m\,\|\mathbf{u}_{\mathrm{nom}} + g\,\mathbf{e}_3\|, \qquad
    \mathbf{R}(\boldsymbol{\Theta}_{\mathrm{cmd}})\,\mathbf{e}_3 = \frac{\mathbf{u}_{\mathrm{nom}} + g\,\mathbf{e}_3}{\|\mathbf{u}_{\mathrm{nom}} + g\,\mathbf{e}_3\|},
\end{align}
with yaw commanded constant, well defined since $U_{\max} < g$ keeps the commanded thrust vector away from zero. It also discharges the operating-region hypothesis of \autoref{lem:iss_attitude}: since $\mathbf{u}_{\mathrm{nom}} + g\,\mathbf{e}_3$ lies in the ball of radius $U_{\max}$ around $g\,\mathbf{e}_3$, the commanded tilt under~\eqref{eq:flat_inversion} never exceeds $\arcsin(U_{\max}/g)$, below $18^\circ$ for the parameters of \autoref{sec:evaluation}. The bound $\Theta_{\max}$ therefore follows from the controller parameters rather than standing as a separate assumption: $\arcsin(U_{\max}/g)$ plus the tracking-error bound of~\eqref{eq:iss_bound} suffices, and stays below $\pi/2$ whenever the inner loop holds the tracking error within the margin $\pi/2 - \arcsin(U_{\max}/g)$. Under~\eqref{eq:flat_inversion} the second term of~\eqref{eq:xi_nl_def} vanishes identically. By \autoref{lem:iss_attitude}, the smoothness of $\mathbf{R}(\cdot)$ with Lipschitz constant $L_R$, and $T/m \leq g + U_{\max}$,
\begin{align}\label{eq:xi_nl_bound}
    \|\boldsymbol{\xi}_{\mathrm{nl}}(t)\| &\leq \bar\xi_{\mathrm{nl}}\big(\|\boldsymbol{\Theta}(0)\|, t\big) \nonumber \\
    &\quad + (g + U_{\max})\, L_R\, \gamma_a\Big(\sup_s \|\dot{\boldsymbol{\Theta}}_{\mathrm{cmd}}(s)\|\Big),
\end{align}
with $\bar\xi_{\mathrm{nl}}$ a decaying transient. In steady state, $\|\boldsymbol{\xi}_{\mathrm{nl}}\|_\infty \leq (g + U_{\max})\, L_R\, \gamma_a(\|\dot{\boldsymbol{\Theta}}_{\mathrm{cmd}}\|_\infty)$.

\begin{theorem}[DMPC Stability under Nonlinear Quadrotor Dynamics]\label{thm:quad_nonlin}
Consider the adaptive DMPC with the nonlinear rigid-body dynamics~\eqref{eq:nonlin_trans}--\eqref{eq:nonlin_rot} and a cascaded controller satisfying \autoref{lem:iss_attitude} with the command extraction~\eqref{eq:flat_inversion}. Under the conditions of \autoref{thm:quad_lin} with $\bar\xi$ replaced by $\|\boldsymbol{\xi}_{\mathrm{nl}}\|_\infty$, the closed loop is practically asymptotically stable about $\{(\bar{\mathbf{p}}_i, \mathbf{0})\}_{i=1}^N$, converging to a neighborhood of radius $\mathcal{O}(\|\boldsymbol{\xi}_{\mathrm{nl}}\|_\infty)$. If the inner loop achieves asymptotically exact tracking ($\|\boldsymbol{\xi}_{\mathrm{nl}}\|_\infty \to 0$ as $\|\mathbf{v}_i\| \to 0$), the stability is asymptotic.
\end{theorem}

\begin{proof}
The feedback deviation is $\mathbf{e}_i = \boldsymbol{\rho}_i(H) + \eta_i + \boldsymbol{\xi}_{\mathrm{nl},i}$, so $\bar e \leq \bar\rho + C\beta + \|\boldsymbol{\xi}_{\mathrm{nl}}\|_\infty$ and \autoref{lem:iss_perturbation} gives~\eqref{eq:iss_quad}, i.e., practical asymptotic stability with ultimate bound $\mathcal{O}(\bar e)$. The structure is identical to \autoref{thm:quad_lin} with $\boldsymbol{\xi}_{\mathrm{nl}}$ in place of $\boldsymbol{\xi}$. Near equilibrium, $\|\mathbf{v}_i\| \to 0$ forces $\|\dot{\boldsymbol{\Theta}}_{\mathrm{cmd}}\| \to 0$, hence $\|\boldsymbol{\xi}_{\mathrm{nl}}\| \to 0$ by~\eqref{eq:xi_nl_bound}. The perturbation is self-extinguishing, and exact asymptotic stability is recovered by the cascade argument for ISS systems~\cite[Ch.~4]{khalil_nonlinear_2002}.
\end{proof}

The rigid-body model~\eqref{eq:nonlin_trans}--\eqref{eq:nonlin_rot} is deliberately the minimal member of its model class: it isolates the two effects that separate a quadrotor from the double integrator, and every higher-fidelity effect enters the analysis similarly. Experimentally identified copter models refine it with first-order motor lag, aerodynamic drag, and gyroscopic rotor torques~\cite{2019quadrotor, lupashin2014platform}. Under the cascaded controller, each of these effects contributes a further bounded term to $\boldsymbol{\xi}_{\mathrm{nl}}$ in~\eqref{eq:xi_nl_def}, the motor lag by the same time-scale separation as the attitude loop and the drag through the speed bound $V_{\max}$, so only the constant $\bar e$ of \autoref{tab:model_classes} changes, not the argument. That real inner loops deliver such residuals, small and certifiable in advance, is documented across control architectures: a cascaded flatness-based design tracks within a precomputed bound under wind disturbances~\cite{2019quadrotor}, incremental nonlinear dynamic inversion holds centimeter-level errors in aggressive flight~\cite{tal2021indi}, and drag-compensated flatness sustains accurate high-speed tracking~\cite{faessler2018drag}. Any of these controllers satisfies the hypothesis of \autoref{lem:iss_attitude} and can serve as the inner loop, with the DMPC as supervisor unchanged.

\begin{table}[t]
\centering
\caption{Feedback perturbation bound across the three model classes.}
\label{tab:model_classes}
\begin{tabular}{lc}
    \toprule
    Model & $\bar e$ \\
    \midrule
    Double integrator & $\bar\rho + C\beta$ \\
    Linearized quadrotor & $\bar\rho + C\beta + \bar\xi$ \\
    Nonlinear quadrotor & $\bar\rho + C\beta + \|\boldsymbol{\xi}_{\mathrm{nl}}\|_\infty$ \\
    \bottomrule
\end{tabular}
\end{table}

The extension relies on the receding-horizon framework rather than on the specific plant. MPC for nonlinear systems is stable provided the optimization is feasible at every step (guaranteed by \autoref{thm:var_guarantees}), stability is enforced through a terminal ingredient~\cite{MAYNE2000789} or a sufficiently long horizon~\cite{ROSTAMI2023100881} (here $H \geq \tilde H_{\min}^{\mathrm{feas}}$), and the system is controllable near the equilibrium. The linearized quadrotor is controllable by construction, and the nonlinear quadrotor inherits local controllability from its linearization, so the framework applies, with our perturbation analysis supplying the explicit convergence rate and parameter bounds.

\begin{corollary}[Variable Horizon under Quadrotor Dynamics]\label{cor:var_quad}
Under the conditions of \autoref{thm:quad_lin} and \autoref{thm:quad_nonlin}, respectively, \autoref{ass:fixed_feas}, and \autoref{ass:hmin} with $H_{\min}^{\mathrm{feas}}$ replaced by $\tilde H_{\min}^{\mathrm{feas}}$, the variable-horizon DMPC of \autoref{def:horizon} retains recursive feasibility and (practical) asymptotic stability under both quadrotor models.
\end{corollary}

\begin{proof}
\autoref{thm:quad_lin} and \autoref{thm:quad_nonlin} establish stability through \autoref{lem:iss_perturbation}, whose only horizon dependence is the residual $\boldsymbol{\rho}_i(H)$ of \autoref{prop:perturbed_lqr}, bounded by $\bar\rho$ uniformly over the admissible band, and the admissible-horizon feasibility of \autoref{lem:braking} holds under $\tilde U_{\max}$. The argument of \autoref{thm:var_guarantees} therefore applies verbatim.
\end{proof}

Across the three model classes, the only quantity that changes is the feedback perturbation $\bar e$ that \autoref{lem:iss_perturbation} absorbs. \autoref{tab:model_classes} lists the three bounds. Everything else carries over: the critical density $N_{\mathrm c}^{\mathrm v}$ and the variable-horizon results of \autoref{sec:theory} are inherited unchanged, and the critical growth rate retains its form with the tightened acceleration reserve, $\tilde\alpha_{\mathrm c} = 2r_{\min}\tilde U_{\max}/V_{\max}^2$. Stability is asymptotic for the double integrator and the linearized quadrotor. For the nonlinear quadrotor it is practical, becoming asymptotic as the inner-loop residual extinguishes near equilibrium (\autoref{thm:quad_nonlin}).
\section{Evaluation}\label{sec:evaluation}

We evaluate whether the conflict-predictive variable horizon achieves efficiency gains without compromising the guarantees of \autoref{sec:theory}, and how it positions the drones in relation to the contact boundary. We compare three horizon strategies under the same distributed controller: a short fixed horizon held at the minimum, $H_{\min}$, a long fixed horizon held at the maximum, $H_{\max}$, and a conflict-predictive variable horizon ranging over $[H_{\min}, H_{\max}]$. This section describes the benchmark (\autoref{ssec:setup}), then examines the cost of computing the variable horizon (\autoref{ssec:cost}), how closely it allows drones to approach one another (\autoref{ssec:separation}), and the regime in which each effect dominates (\autoref{ssec:regime}).

\subsection{Benchmark Setup}\label{ssec:setup}

The scenarios are antipodal swaps, in which every drone flies to the diametrically opposite point through the center of the airspace. This produces a dense, conflict-heavy crossing at mid-flight, preceded and followed by clear airspace. This is exactly the regime that the variable horizon targets. The matrix spans drone counts $N \in \{2, 4, 8\}$ and tight ($5^3\,\mathrm{m}^3$) and open ($20^3\,\mathrm{m}^3$) cubic airspace, giving six scenarios. The safety model is the adaptive radius of the framework with a minimum radius of $r_{\min} = 0.4\,$m and growth rate $\alpha = 0.25$. The speed and acceleration limits are $V_{\max} = 3.0\,$m/s and $U_{\max} = 3.0\,$m/s$^2$, and the sampling time is $\Delta t = 0.1\,$s. These values satisfy the stability bound $\alpha < \alpha_{\mathrm{c}} = 2 r_{\min} U_{\max}/V_{\max}^2 = 0.27$, so the experiments instantiate the system analyzed in \autoref{sec:theory}. \autoref{fig:trajectories} gives an overview of the resulting closed-loop motion. The swaps are fully three-dimensional. In the open scenarios the drones hold their direct lines and deviate only around the central crossing, whereas the tight scenarios keep them on detours for most of the flight.

\begin{figure}[t]
    \centering
    \includegraphics[width=\columnwidth]{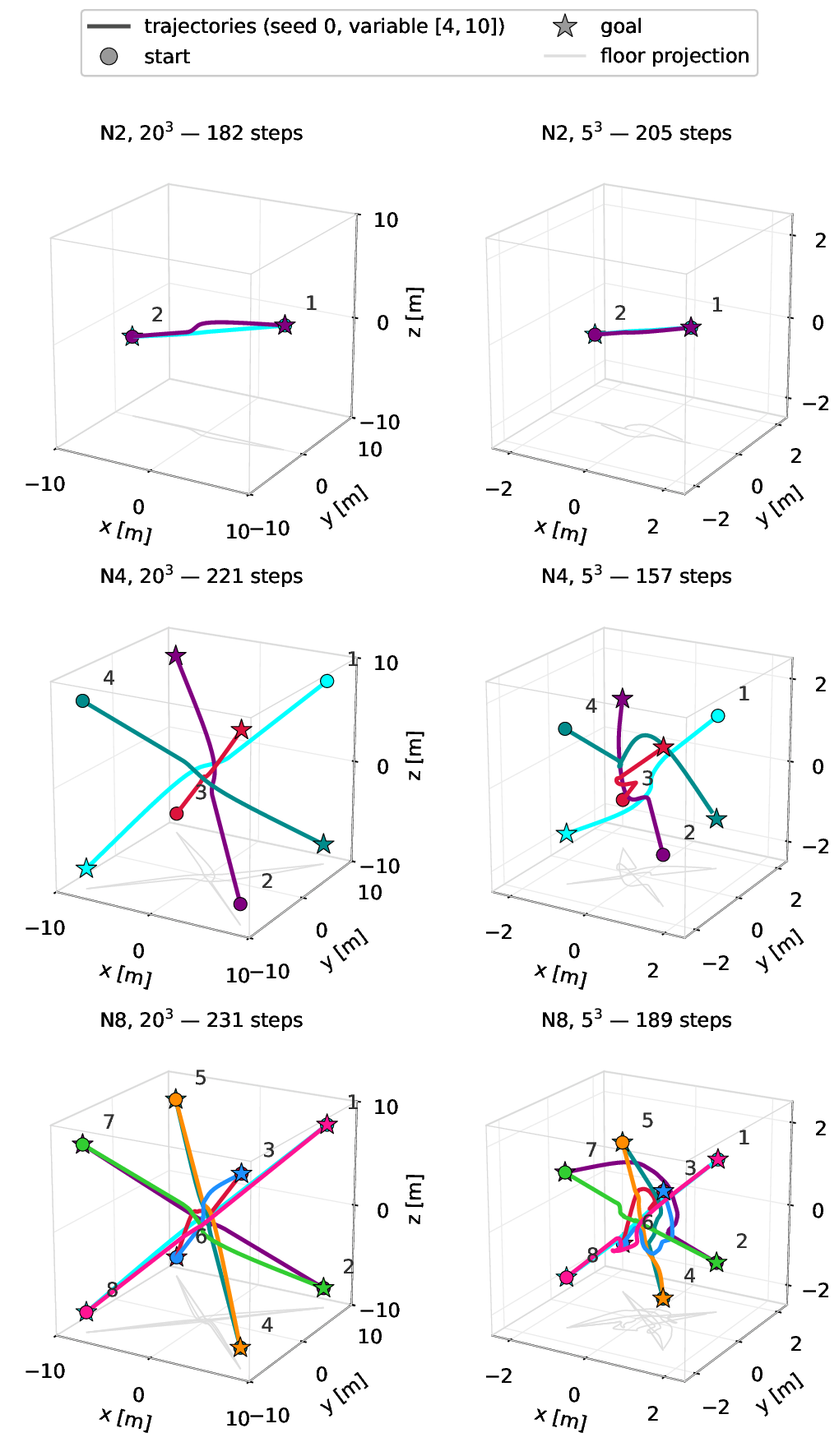}
    \caption{Closed-loop trajectories under the variable horizon for the six benchmark scenarios, in open ($20^3\,\mathrm{m}^3$, top) and tight ($5^3\,\mathrm{m}^3$, bottom) airspace. Circles mark the start positions, stars the goals, and the gray curves on the floor are the horizontal projections. One representative seed is shown per scenario, since overlaying all $20$ would obscure the individual paths.}
    \label{fig:trajectories}
\end{figure}

The variable horizon ranges from $H_{\min} = 4$ to $H_{\max} = 10$, and the two fixed baselines run at its endpoints. With these parameters, the feasibility bound~\eqref{eq:hmin_feas} evaluates to $H_{\min}^{\mathrm{feas}} = \lceil\tfrac{1}{\Delta t}\sqrt{2r_{\min}\alpha/U_{\max}}\rceil = 3$. Thus, $H_{\min} = 4$ meets \autoref{ass:hmin} with a margin of one step. The conflict detector uses the history length $L = 50$, the decay constant $c_{\mathrm{d}} = 0.3$, and the speed lower bound $\nu = 0.5$. The funnel lower bound is the neighbor's adaptive radius evaluated at the fitted speed. Each scenario and strategy is run over $20$ random seeds with a per-run budget of $500\,$s, giving $360$ runs in total. Within a seed, the initial geometry is identical across strategies. We report a cost sweep that records per-step solver time, prediction time, and total computation, and separation traces that record the minimum pairwise distance over time. All timings are wall-clock under fixed parallel execution, so the evidence lies in the ratios within a scenario, not in absolute seconds.

\subsection{Computational Cost}\label{ssec:cost}

\autoref{fig:cost_curves} decomposes the per-step cost across the six scenarios, as means over the $20$ seeds: solving at the short horizon ($H_{\min}$), solving at the long horizon ($H_{\max}$), solving under the variable horizon, and the linear prediction of \autoref{ssec:prediction}. The prediction costs $0.04$ to $0.7$ ms per step, two to four orders of magnitude below the solver time at any horizon and under $0.3\%$ of the variable horizon's own per-step solver time. The long-horizon solver time rises steeply with conflict intensity, from about $0.1\,$s per step in the two-drone scenarios to $10.7\,$s in the eight-drone open scenario, where it exhausts the $500\,$s budget within the first fifty steps.

\begin{figure}[t]
    \centering
    \includegraphics[width=\columnwidth]{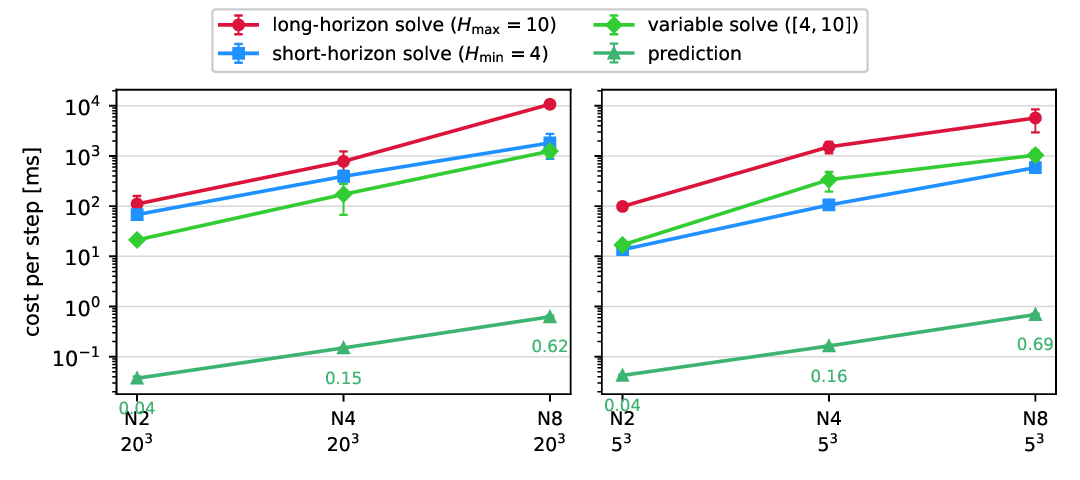}
    \caption{Per-step solver and prediction cost across the six scenarios, log scale, as means over $20$ seeds per scenario and strategy.}
    \label{fig:cost_curves}
\end{figure}

The variable horizon spends almost all steps at its lower bound: its mean selected horizon lies between $4.1$ and $4.5$ across the scenarios, over the band $[4, 10]$. Its per-step solver cost tracks the short-horizon curve and lies $4.5\times$ to $8.6\times$ below the long horizon's in every scenario, where the eight-drone open value is a lower bound because the long horizon never passes the crossing. Since the horizon ratio is only $2.5$, these per-step ratios confirm the superlinear growth of the solver cost in the horizon length.

\autoref{fig:cost} reports the total computation. The variable horizon undercuts the long fixed horizon by $2.0\times$ to $2.7\times$ in every scenario in which the latter finishes. In the eight-drone open scenario the long horizon exceeds the budget in all $20$ seeds, while the variable horizon completes every seed in $290\,$s on average. In the open scenarios it even undercuts the short fixed horizon, $38$ versus $98\,$s at four drones, because raising the horizon at the crossing resolves the conflict earlier. The variable horizon completes the mission in fewer steps than the short one, $220$ versus $248$, and spends fewer of them in expensive active constraints, which lowers the mean per-step solver time to $172$ versus $393\,$ms. The long fixed horizon needs only $129$ steps, but at far higher per-step cost. In the tight scenarios, where the confined space keeps conflicts active throughout, the variable horizon costs up to $2.7\times$ more than the short fixed horizon, the price of its longer planning windows in this case.

\begin{figure}[t]
    \centering
    \includegraphics[width=\columnwidth]{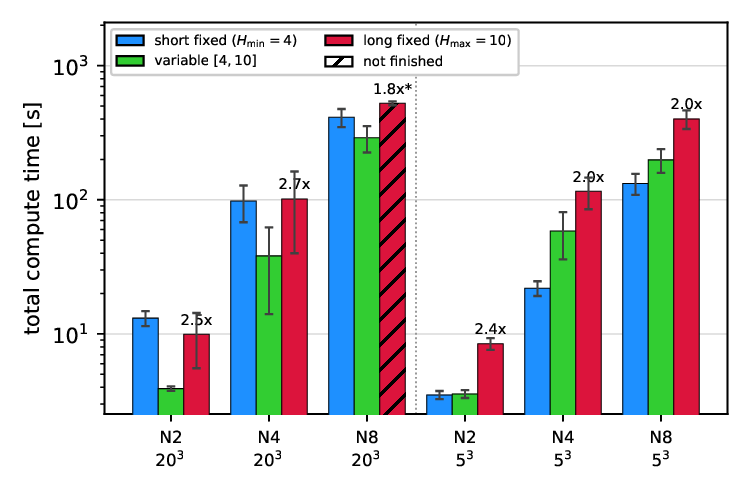}
    \caption{Total computation per scenario for the three strategies, log scale, as means over $20$ seeds. Labels give the total-time ratio of the long fixed horizon to the variable horizon. The asterisk marks the eight-drone open scenario, where the long fixed horizon exceeds the $500\,$s budget in every seed.}
    \label{fig:cost}
\end{figure}

\subsection{Separation and Safety}\label{ssec:separation}

The variable horizon never violates the separation constraint: over all $120$ of its runs, the minimum pairwise distance stays at or above the contact boundary $2r_{\min} = 0.8\,$m. The long fixed horizon likewise maintains the boundary wherever it completes. The short fixed horizon does not: in the eight-drone open scenario, $2$ of $20$ runs breach the boundary, dipping to $0.776\,$m at worst. In the tight scenarios all strategies ride the boundary through the crossing, with grazing contact at exactly $0.80\,$m. \autoref{tab:strategies} lists the worst-case closest approach per scenario over the $20$ seeds. The dagger marks the boundary breach, and the asterisk marks the long horizon's eight-drone open value, which covers only the pre-crossing phase because of the timeout.

\autoref{fig:separation} traces the minimum pairwise distance in the two four-drone scenarios, as seed means with min--max envelopes. In the open scenario the pattern is monotone in the horizon. The short fixed horizon reacts late: it stays near the boundary for the longest time, $32$ steps below $1.2\,$m on the mean trace, with a worst-case closest approach of $0.81\,$m. The long fixed horizon anticipates early and stays farthest away, $1.07\,$m at worst and no step below $1.2\,$m. The variable horizon dips almost as low as the short horizon at the densest instant, $0.86\,$m at worst, but only briefly, $3$ steps below $1.2\,$m on the mean trace. In the tight scenario the confined space forces all three strategies onto the boundary, and they differ only in cost (\autoref{ssec:cost}).

This behavior shows where the assumptions of \autoref{sec:theory} become tight. The guarantees are conditional on per-step solver feasibility (\autoref{ass:fixed_feas}), and the short fixed horizon, although it meets the feasibility bound $H_{\min}^{\mathrm{feas}} = 3$ with one step to spare, operates at exactly that margin: at eight drones its late reaction leaves conflicts the solver can no longer resolve within four steps, producing the isolated breaches. The variable horizon shares that lower bound yet avoids every breach, because lifting the horizon at a predicted crossing keeps the solver away from the margin. Prediction quality thus governs efficiency within the certified band (\autoref{cor:robustness}), and in practice it also eases the burden on the feasibility assumption on which safety rests.

\begin{figure}[t]
    \centering
    \includegraphics[width=0.85\columnwidth]{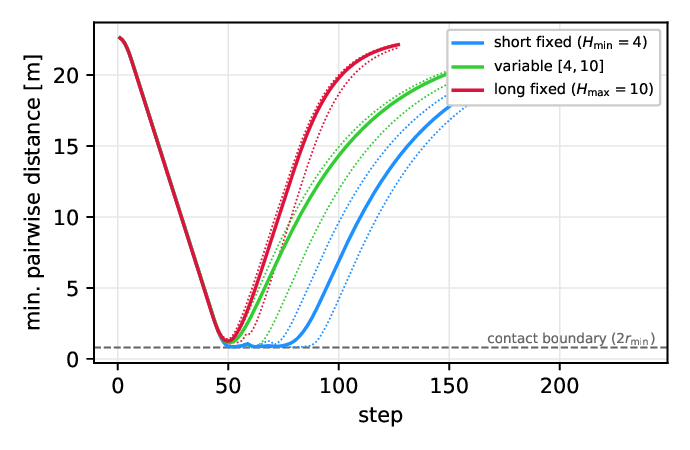}\\[2pt]
    \includegraphics[width=0.85\columnwidth]{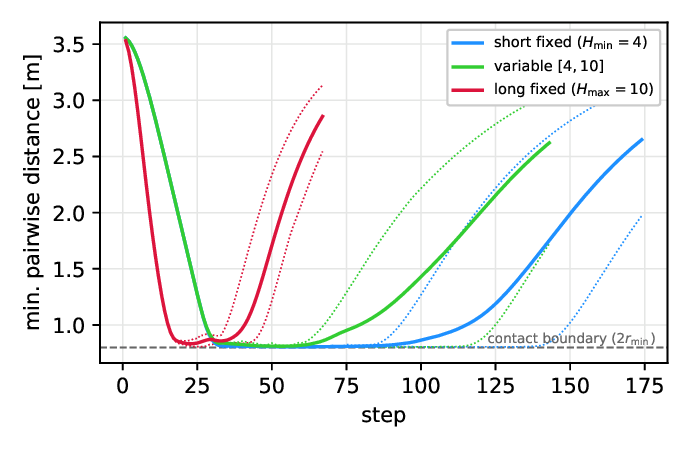}
    \caption{Minimum pairwise distance in the four-drone open scenario (top) and tight scenario (bottom): seed mean per strategy with min--max envelope over $20$ seeds. The contact boundary is $2r_{\min} = 0.8\,$m.}
    \label{fig:separation}
\end{figure}

\begin{table}[t]
\centering
\caption{Worst-case closest approach [m] over $20$ seeds per scenario.}
\label{tab:strategies}
\begin{tabular}{l c c c}
    \toprule
    Scenario & short fixed & variable & long fixed \\
    \midrule
    $N{=}2$, $20^3$ & $0.84$ & $1.12$ & $1.34$ \\
    $N{=}4$, $20^3$ & $0.81$ & $0.86$ & $1.07$ \\
    $N{=}8$, $20^3$ & $0.78^\dagger$ & $0.81$ & $0.83^*$ \\
    $N{=}2$, $5^3$  & $0.80$ & $0.80$ & $0.80$ \\
    $N{=}4$, $5^3$  & $0.80$ & $0.80$ & $0.81$ \\
    $N{=}8$, $5^3$  & $0.80$ & $0.80$ & $0.81$ \\
    \bottomrule
\end{tabular}
\end{table}

\subsection{Operating Regime}\label{ssec:regime}

The outcome of the comparison depends on the baseline and on the operating regime. Against the long fixed horizon, the variable horizon dominates everywhere: it is $2.0\times$ to $2.7\times$ faster in total wherever the long horizon finishes, completes the eight-drone open swap that the long horizon cannot, and concedes separation margin only where margin is abundant. Against the short fixed horizon, the picture splits by airspace. In the open scenarios, the common case in the large-volume operations that motivate the scheme, the variable horizon is both faster and safer: it undercuts the short horizon's total time and maintains the boundary that the short horizon breaches at eight drones. In the tight scenarios every strategy sits on the boundary and the short horizon is cheaper, so the variable horizon's premium buys anticipation that the confined geometry cannot convert into margin. It is thus the only strategy that is tractable and boundary-respecting across the whole benchmark matrix.

\section{Discussion and Conclusion}\label{sec:conclusion}

We present a conflict-predictive variable horizon for distributed MPC. Each drone sizes its planning window to cover its farthest predicted conflict, within a fixed band. The horizon is the only quantity exposed to an otherwise unchanged controller, so the cost of a long horizon is paid only when a conflict is anticipated. We prove that recursive feasibility and asymptotic stability are preserved for every horizon in that band under a single computable condition on its lower bound (\autoref{ass:hmin}, \autoref{thm:var_guarantees}), and that the same perturbation argument carries the guarantees to the linearized quadrotor and, as practical stability, to the full nonlinear one (\autoref{cor:var_quad}). On dense antipodal-swap benchmarks the variable horizon undercuts a long fixed horizon in both per-step and total computation, completes swaps that the long horizon cannot finish within the time budget, and holds separation where a short fixed horizon of comparable per-step cost loses it. To the best of our knowledge, this is the first variable-horizon DMPC whose feasibility and stability guarantees are independent of the prediction that drives the horizon.

The horizon thereby turns from a fixed budget into a resource allocated on demand. The capacity analysis we build on treats it as exactly that budget, in the sense that a horizon $H$ resolves up to $\kappa(H)$ sequential conflicts~\cite{drones10020139}. The variable horizon spends it where conflicts are predicted and conserves it elsewhere, without weakening the density limits, which are purely geometric. Unlike learned and chance-constrained schemes, in which the forecast enters the avoidance constraint and prediction accuracy is therefore safety-critical, our construction keeps the predictor out of the safety argument altogether. A better predictor improves only efficiency, and a worse one cannot invalidate the certified band (\autoref{cor:robustness}). Beyond quadrotors, the construction covers any vehicle whose translational dynamics reduce to a double integrator perturbed by a bounded residual under a stabilizing inner loop, since the combined perturbation $\bar e$ need only meet the dissipation condition of \autoref{lem:iss_perturbation}.

The scheme has limitations, and each points to a concrete extension. The linear predictor loses accuracy against strongly maneuvering neighbors, which the funnel mitigates but does not remove. A learned predictor at the same funnel interface would sharpen anticipation without touching the guarantees. At the densest crossing the horizon collapses to its lower bound and the drones pass closer to the contact boundary than under a long fixed horizon, although they never breach it and spend far less time near it than under a short one. A maneuver-reserve bound that lifts the horizon while a conflict is detected would keep them farther from the boundary without raising $H_{\min}$. The nonlinear extension assumes an inner loop of sufficient bandwidth and an attitude error small enough for the linearization of $\mathbf{R}(\boldsymbol{\Theta})$ to stay valid, which excludes aggressive aerobatic flight. Large excursions would require a geometric controller on $SO(3)$ with almost-global stability~\cite{lee2010geometric}. Finally, the evaluation is in simulation, and its wall-clock timings are informative as ratios within a scenario rather than as absolute figures, so hardware experiments remain to be done.

The key question we leave open is stated in \autoref{rem:beta_relax}. The variable horizon appears to weaken the coupling between neighboring subproblems, because it resolves conflicts earlier and with less constraint activation. A proof of that link would turn the contraction condition from something the variable horizon preserves into something it improves.

\section*{Funding}

The research project was funded by ``The Ministry of the Environment, Nature Conservation and Transport of the State of North Rhine-Westphalia'' in Germany and co-financed by the European Union for the research project ``SIDDA - Sustainable Intermodal Drone Delivery Airline'' with grant number IN-ML-1-013b.
\section*{Conflict of Interest}
The authors declare no conflicts of interest. They have no known competing financial interests or personal relationships that could have appeared to influence the work reported in this manuscript. The funding source had no role in the design of the study, in the collection, analysis, or interpretation of data, in the writing of the manuscript, or in the decision to submit it for publication.

\bibliographystyle{IEEEtran}
\bibliography{bib}

\end{document}